\documentclass[12pt]{article}
\usepackage{amsmath}
\usepackage{times}
\usepackage{graphicx}
\usepackage{multirow}
\usepackage[authoryear]{natbib}
\usepackage{rotating}
\usepackage{bbm}
\usepackage{latexsym}

\newcommand{\captionfonts}{\normalsize}

\makeatletter  
\long\def\@makecaption#1#2{%
  \vskip\abovecaptionskip
  \sbox\@tempboxa{{\captionfonts #1: #2}}%
  \ifdim \wd\@tempboxa >\hsize
    {\captionfonts #1: #2\par}
  \else
    \hbox to\hsize{\hfil\box\@tempboxa\hfil}%
  \fi
  \vskip\belowcaptionskip}
\makeatother   

\usepackage{amsthm}
\newtheorem{theorem}{Theorem}
\newtheorem{corollary}{Corollary}
\newtheorem{lemma}{Lemma}
\usepackage{bbm}
\usepackage{url}
\usepackage{amssymb}
\usepackage{bm}
\usepackage{subfigure}
\usepackage{comment}

\usepackage{hyperref}
\hypersetup{colorlinks,
            linkcolor=blue,
            citecolor=blue,
            urlcolor=magenta,
            linktocpage,
            plainpages=false}
\usepackage{cite}

\allowdisplaybreaks[1]

\usepackage{xparse}
\DeclareDocumentCommand{\comb}{m m}{{}_{#1}C_{#2}}

\makeatletter
\newcommand*\rel@kern[1]{\kern#1\dimexpr\macc@kerna}
\newcommand*\widebar[1]{%
  \begingroup
  \def\mathaccent##1##2{%
    \rel@kern{0.8}%
    \overline{\rel@kern{-0.8}\macc@nucleus\rel@kern{0.2}}%
    \rel@kern{-0.2}%
  }%
  \macc@depth\@ne
  \let\math@bgroup\@empty \let\math@egroup\macc@set@skewchar
  \mathsurround\z@ \frozen@everymath{\mathgroup\macc@group\relax}%
  \macc@set@skewchar\relax
  \let\mathaccentV\macc@nested@a
  \macc@nested@a\relax111{#1}%
  \endgroup
}
\makeatother

\begin{document}
\hspace{13.9cm}1

\ \vspace{20mm}\\

{\LARGE Q$\&$A Label Learning}

\ \\
{\bf \large Kota Kawamoto$^{\displaystyle 1}$ and Masato Uchida$^{\displaystyle 1}$}\\
{$^{\displaystyle 1}$Waseda University, Tokyo, Japan.}
%


\thispagestyle{empty}
\markboth{}{NC instructions}
\ \vspace{-0mm}\\
%
\begin{center} {\bf Abstract} \end{center}
Assigning labels to instances is crucial for supervised machine learning.
In this paper, we proposed a novel annotation method called Q\&A labeling, which involves a question generator that asks questions about the labels of the instances to be assigned, and an annotator who answers the questions and assigns the corresponding labels to the instances.
We derived a generative model of labels assigned according to two different Q$\&$A labeling procedures that differ in the way questions are asked and answered.
We showed that, in both procedures, the derived model is partially consistent with that assumed in previous studies.
The main distinction of this study from previous studies lies in the fact that the label generative model was not assumed, but rather derived based on the definition of a specific annotation method, Q\&A labeling.
We also derived a loss function to evaluate the classification risk of ordinary supervised machine learning using instances assigned Q$\&$A labels and evaluated the upper bound of the classification error.
The results indicate statistical consistency in learning with Q$\&$A labels.

\section{Introduction}

In standard supervised machine learning, an ordinary label is typically assigned to each instance to represent the single class to which it belongs.
However, depending on the difficulty in correctly identifying the class label for an instance and the annotator's ability, assigning an ordinary label may be challenging.
To address these challenges, Ishida et al.\ introduced an annotation concept in which each instance is assigned a complementary label representing the class to which it does not belong \citep{ishida2017nips}.
Complementary labels can be assigned using various methods.
For example, a single class that is considered the least suitable for the given instance can be assigned as a complementary label.
Another possible method is to assign one of the arbitrary classes that is considered unsuitable for the given instance as a complementary label.
These differences in annotation methods can affect the nature of the assigned complementary labels.
However, Ishida et al.\ assumed a probabilistic model for generating complementary labels without specifying a specific annotation method \citep{ishida2017nips}.
This assumption implies a situation in which ``other than the class to which the instance belongs, all other classes are selected with equal probability as classes to which the instance does not belong''.
However, whether an annotation method that is consistent with this situation can be implemented remains unclear.

Katsura et al.\ proposed a scenario for annotating candidate labels that represent multiple classes to which an instance is likely to belong \citep{katsura2020bridging}.
This scenario is a generalization of the annotation concept for complementary labels.
As with annotating complementary labels, various methods are possible for annotating candidate labels.
However, no specific annotation method for complementary labels has been presented yet.
The generative model for the candidate label introduced by Katsura et al.\ \citep{katsura2020bridging} is an extension of that for the complementary label assumed by Ishida et al.\ \citep{ishida2017nips}; thus, its feasibility remains unclear.
The theoretical properties of learning with complementary and candidate labels were discussed in \citep{ishida2017nips,katsura2020bridging} based on the implicit assumption of a generative model. However, the validity of this assumption remains unclear.

In this paper, we propose an annotation method that enables annotation even in cases where assigning an ordinary label is difficult, and derive a generative model realized by this method.
We also show that this generative model coincides with the models assumed in \citep{ishida2017nips,katsura2020bridging} under certain conditions.
This provides a theoretical basis for the validity of assumptions regarding a generative model that has not been discussed in \citep{ishida2017nips,katsura2020bridging}.
Furthermore, we derive an upper bound for classification errors in learning using labeled training data annotated with our proposed annotation method, and demonstrate that this learning has statistical consistency.

The proposed annotation method in this paper includes a question generator that asks questions about the labels to be assigned, and an annotator who answers those questions and performs annotations.
The question generator presents the annotator with a set of question classes, which is a subset of the total class set (i.e., the set of all possible classes), and asks a question based on this set.
The annotator answers the question by referring to the question class set.
Hereafter, the proposed annotation method is called Q$\&$A labeling, and the label assigned by Q$\&$A labeling is called the Q$\&$A label.
There are two types of methods for Q$\&$A labeling: which-one and is-in.
In the which-one-type, if the question class set presented by the question generator contains a class that the annotator considers the most suitable, the annotator selects that class as the answer, and it is assigned as the label.
If the question class set does not contain the classes that the annotator considers the most suitable, all classes in the complement of the question class set are assigned as labels.
In the is-in-type, the annotator answers whether the question class set presented by the question generator contains the class that the annotator considers the most suitable.
If the question class set contains classes that the annotator considers the most suitable, all the classes in the question class set are assigned as labels.
Conversely, if the question class set does not contain the classes that the annotator considers the most suitable, all the classes in the complement of the question class set are assigned as labels, as in the which-one-type.


The structure of this paper is as follows.
In Section \ref{sec:related_work}, we provide an overview of related research and explain the positioning of our study.
In Section \ref{sec:preliminary}, we formulate the multi-class classification problem with ordinary labels as a preparation for our study. 
We also discuss the label generative model used in learning with candidate labels and its interpretation.
In Section \ref{sec:formulation}, we formulate learning with Q\&A labels and derive the label generative model, explaining its interpretation.
In Section \ref{sec:analysis}, we analyze learning with Q\&A labels and derive an upper bound on classification error.
In Section \ref{sec:experiment}, we evaluate the performance of learning with Q\&A labels through experiments and discuss the validity of theoretical analysis.
Section \ref{sec:conclusion} presents the conclusion of this paper.
\section{Related Work}
\label{sec:related_work}

Several studies have attempted to reduce the difficulty of annotating ordinary labels by assigning labels that are not necessarily from a single class.
For example, the concept of partial labels, wherein an instance belongs to a single class but multiple classes including that class are assigned as candidates labels \citep{cour2011learning,feng2020provably,pmlr-v139-wen21a}, and the concept of pairwise similarity, wherein pairs of instances are assigned labels indicating whether they are similar or not \citep{bao2018classification,shimada2021classification}, have been proposed. 
The concept of assigning a single class to which an instance does not belong as a label (complementary label) aims to reduce the difficulty of annotating ordinary labels \citep{ishida2017nips,ishida2019icml}.

These studies assumed a label generative model when discussing the properties of learning algorithms.
For example, the generative model of partial labels used in \citep{cour2011learning} assumes that the generated labels always contain the correct class.
In addition, in \citep{feng2020provably} it was assumed that classes other than the correct class were uniformly included.
In \citep{pmlr-v139-wen21a}, this uniformity assumption was removed; however, it is still assumed that the generated labels always contain the correct class.
In studies on pairwise similarity, it is assumed that the bias in discriminating similar instances, which should be different for each annotator, follows the same probability model \citep{bao2018classification, shimada2021classification}.
In studies on complementary labels, it is assumed that all classes, except the class to which the instance belongs, are selected with equal probability as the class to which the instance does not belong \citep{ishida2017nips,ishida2019icml}.
However, specific annotation methods to fulfill these assumptions have not been discussed, and the validity of the label generative model thus remains unclear.


Some studies have been conducted to improve the accuracy of learning when instances with complementary labels are used as training data.
Feng et al.\ introduced a concept inspired by the complementary label, where multiple classes that an instance does not belong to are assigned as labels (multi-complementary label) \citep{pmlr-v119-feng20a}.
Cao et al.\ proposed a method for learning from instances assigned multi-complementary labels and unlabeled instances \citep{Cao2020MultiComplementaryAU}.
Katsura et al.\ introduced the concept of the candidate label, which is a generalization of the ordinary label and complementary label \citep{katsura2020bridging,katsura2021candidate}.
Ishiguro et al. discussed a case in which the correct class was incorrectly assigned as a complementary label \citep{Ishiguro2022noisycon}.
However, all label generative models that were used in these studies are based on the assumptions made by Ishida et al.
Yu et al.\ removed the unnatural assumption in complementary labels that all other classes, except the class to which the instance belongs, are selected with equal probability as the class to which the instance does not belong \citep{yu2018learning}.
However, it is assumed that the probability of a class, other than the correct class being assigned as a complementary label, is equal, regardless of the instance.

In contrast, the implementation method of the proposed Q\&A labeling in this paper is specifically defined, thus ensuring the feasibility of assigning labels based on Q\&A labeling.
Unlike the aforementioned existing studies, which assumed label generative models, this study conducts a theoretical analysis by deriving a generative model of Q$\&$A label according to a clearly specified method of executing Q$\&$A labeling.
Remarkably, the derived generative model of the Q\&A label demonstrates that the concept of the Q$\&$A label is partially equivalent to that of the candidate label \citep{katsura2020bridging,katsura2021candidate}, which is a generalization of the complementary label \citep{ishida2017nips,ishida2019icml}. 
This indicates that the feasibility of the concept of the candidate label is also supported.

\section{Preliminary}\label{sec:preliminary}
In this section, we first formulate a multi-class classification problem using ordinary labels.
We then overview the generative model assumed in the concept of the candidate label introduced by Katsura et al.\ \citep{katsura2020bridging,katsura2021candidate}.

\subsection{Multiclass Classification with Ordinary Labels}
General supervised machine learning for multi-class classification problems is performed under a concept in which each instance is assigned a single class label (ordinary label) that the annotator considers the most suitable.
In this concept, the probability that a pair of instance $x \in \mathcal{X}$ and its assigned ordinary label $y \in \mathcal{Y}$, $(x, y)$, is generated is denoted by $P(x,y):=\Pr\{X=x,Y=y\}$.
Here, $X$ is a random variable representing the instance, and $Y$ is a random variable representing the ordinary label assigned to the instance.
In addition, $\mathcal{X}$ is the set of possible instances, and $\mathcal{Y} := \{1,.... ,K\}$ is the total class set (i.e., the set of all possible classes).
The probability $P(x,y)$ of a pair $(x,y)$ being generated is calculated as the product of two probabilities: $P(x):=\Pr\{X=x\}$, which is the probability of generating instance $x$, and $P(y|x):=\Pr\{Y=y|X=x\}$, which is the probability of generating ordinary label $y$ given instance $x$ . Therefore, $P(x,y) = P(x)P(y|x)$.

Let $f:\mathcal{X}\rightarrow\mathcal{Y}$ be the hypothesis that maps an instance $x\in\mathcal{X}$ to class $y\in\mathcal{Y}$ to which it belongs.
Let $g_y:\mathcal{X}\rightarrow \mathbb{R}$ be the binary discriminant function that classifies class $y$.
Then, hypothesis $f$ can be defined as $f(x)=\arg\max_{y\in\mathcal{Y}}g_y(x)$.
The learning objective is to find the hypothesis $f$ that minimizes the classification risk $R(f)$, defined as the expected value of the loss function $\mathcal{L}$ with respect to $P(X,Y)$ as follows:
\begin{align}
R(f) = \mathbb{E}_{P(X,Y)}[\mathcal{L}(f(X), Y)],
\label{eq:pred_err}
\end{align}
where $\mathbb{E}_{P(X,Y)}$ denotes an expectation with respect to $P(X,Y)$. 
For a set of hypotheses $\mathcal{F}\subset\{h: \mathcal{X}\rightarrow \mathbb{R}^K\}$, the Bayesian hypothesis that minimizes classification risk is given as $f^*:= \arg\min_{f\in\mathcal{F}} R(f)$.

\subsection{Generative Model of Candidate Labels}
The concept of candidate labels introduced by Katsura et al.\ assumes that the annotator assigns a label to multiple classes to which each instance belongs \citep{katsura2020bridging}.
The number of classes selected as the label to be assigned, $N$ ($1 \leq N \leq K-1$), is determined identically regardless of the instance.
In this concept, let $\widebar{P}_{N}(x,\widebar{y}):=\Pr\{X=x,\widebar{Y}=\widebar{y}\}$ denote the probability of a pair of instances $x \in \mathcal{X}$ and the corresponding candidate label of size $N$, $\widebar{y} \in \mathfrak{B}_N(\mathcal{Y})$, is generated.
Here, $X$ is a random variable representing the instance and $\widebar{Y}$ is a random variable representing the candidate label assigned to the instance.
In addition, $\mathcal{X}$ is the set of possible values for the instance and $\mathfrak{B}_N(\mathcal{Y})$ is the power set of size $N$ of the total class set $\mathcal{Y}$, which is the set of all possible candidate labels.
The probability that a pair $(x, \widebar{y})$ is generated, $\widebar{P}_N(x,\widebar{y})$, can be decomposed into the product of the generation probability of instance $x$, $P(x):=\Pr\{X=x\}$, and the generation probability of the candidate label $\widebar{y}$ conditioned on a given instance $x$, $\widebar{P}_N(\widebar{y}|x):=\Pr\{\widebar{Y}=\widebar{y}|X=x\}$.
Katsura et al.\ assume that $\widebar{P}_N(\widebar{y}|x)$ has the following form:
\begin{align}
    \widebar{P}_N(\widebar{y}|x)
    = \frac{1}{\comb{K-1}{N-1}}\sum_{y\in \widebar{y}}P(y|x)
    \label{eq:CandP}
\end{align}
Generating an ordinary label is equivalent to generating a candidate label when $N=1$ as $\widebar{P}_1(\{y\}|x)=P(y|x)$ holds.
Additionally, the generative model of the complementary label introduced by Ishida et al.\ \citep{ishida2017nips,ishida2019icml} is equivalent to the generative model $\widebar{P}_{K-1}(x,\widebar{y})$ for a candidate label of size $K-1$.

    \label{eq:arbitry_loss}


\subsection{Receiver's Perspective of Candidate Labels}\label{sec:recipients_perspective}
Let us consider the amount of information regarding the correct labels obtained from the data with candidate labels generated according to Eq.\ \eqref{eq:CandP}.
We define a probability $\Pr\{\widehat{Y} = \alpha|X=x\}$, which represents the degree of confidence, where the one given the the data considers that the assumed ground-truth label $\widehat{Y}$ for the instance $x$ is $\alpha\in\mathcal{Y}$. 
Now, we assume that a set of candidate labels of size $N$, $\widebar{Y}$, is provided to the instance $x$ by the annotator.
Furthermore, we assume that the one given the data does not have any information about the ground-truth label, except for the given set of candidate labels $\widebar{Y}$.
That is, we assume that the degree of confidence, where the assumed ground-truth label $\widehat{Y}$ for the instance $x$ is $\alpha\in\mathcal{Y}$, depends on whether $\alpha \in \widebar{Y}$ or $\alpha \not\in \widebar{Y}$. Thus, $\Pr\{\widehat{Y}=\alpha|X=x, \alpha \in \widebar{Y}\}=1/N$ and $\Pr\{\widehat{Y}=\alpha|X=x, \alpha \not\in \widebar{Y}\}=0$ hold.
Subsequently, the following equation holds \citep{katsura2020bridging}.
\begin{align}
\Pr\{\widehat{Y}=\alpha|X=x\} = \beta P(\alpha|x) + (1-\beta)\frac{1}{K},\label{eq:privacy}
\end{align}
where $\beta = (K-N)/\{N(K-1)\}$.

According to Eq.\ \eqref{eq:privacy}, the degree of confidence $\Pr\{\widehat{Y}=\alpha|X=x\}$ (i.e., the information about labels that is available to the one given the data) is represented by a mixture of the distribution $P(\alpha|x)$ and uniform distribution $1/K$, while $P(\alpha|x)$ represents the information about labels originally owned by the annotator.
This equation indicates that receiving data provided with a set of $N$ candidate labels generated from Eq.\ \eqref{eq:CandP} is equivalent to receiving a data provided with an ordinary label to which random noise is added.
The level of random noise $(1-\beta)$ increases as the number of candidate labels $N$ increases.
\section{Formulation of Q\&A Labeling}\label{sec:formulation}
\subsection{Procedure of Q$\&$A Labeling}
Q$\&$A labeling is executed by two entities: a question generator that asks questions about the labels to be assigned and an annotator that answers the questions and makes annotations in response to the questions.
Regarding the process of executing ordinary labeling, it is assumed that for each instance to annotate, the annotator is able to infer a single class to which it belongs.
In the following section, the process of executing the Q$\&$A labeling is described (see Fig. \ref{fig:error}).

\begin{figure*}[t]\centering
    \subfigure[Ordinary Labeling]{
        \includegraphics[keepaspectratio, scale=0.4]{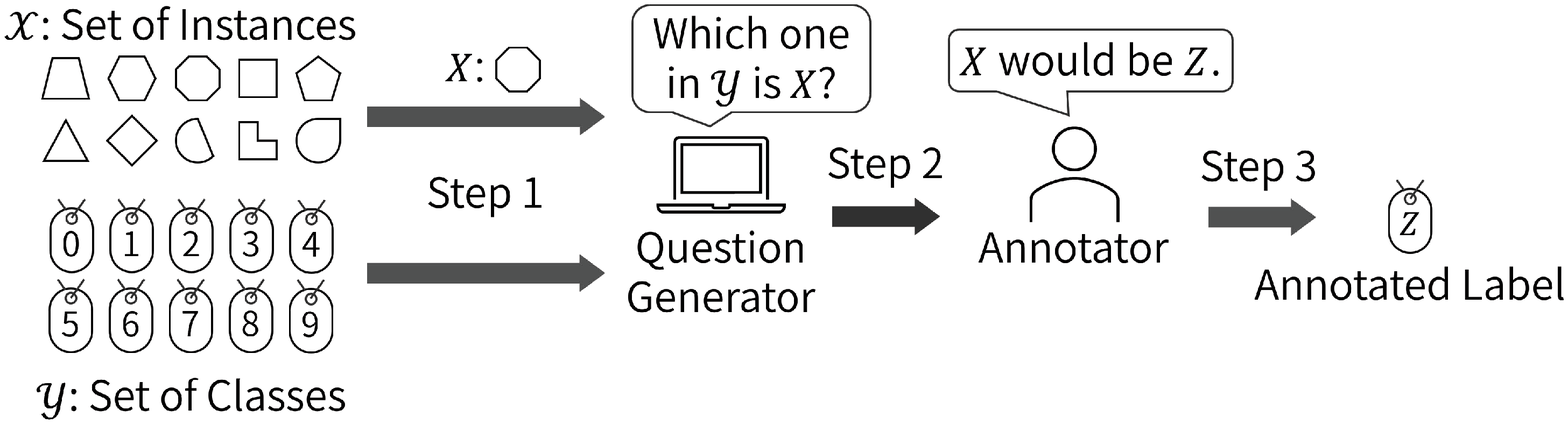}}
    \subfigure[Which-one-type Q$\&$A labeling ($I$ = 3)]{
	\includegraphics[keepaspectratio, scale=0.4]{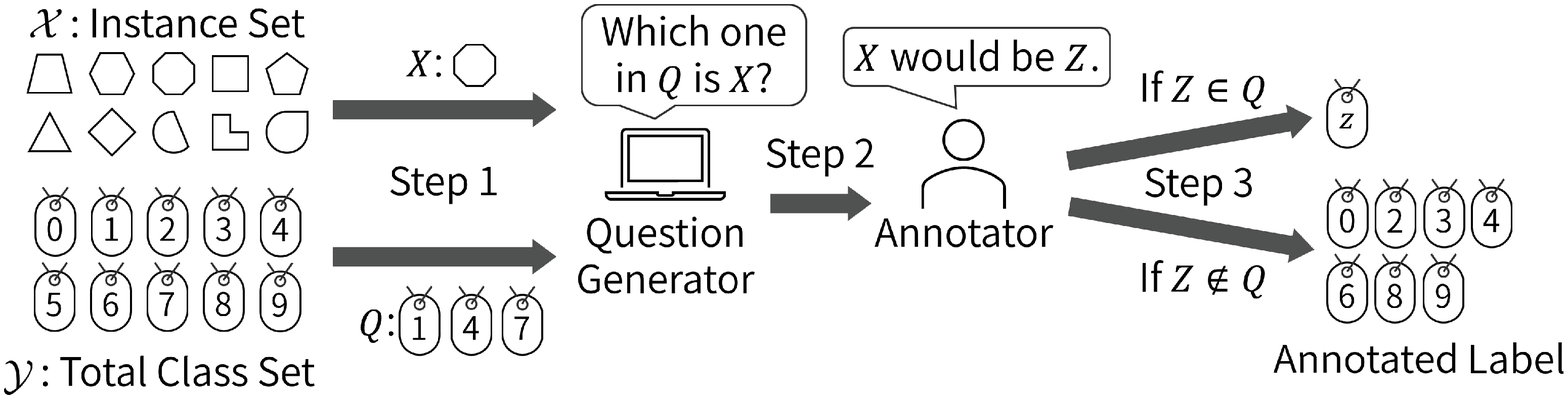}}
    \subfigure[Is-in-type Q$\&$A labeling ($I$ = 3)]{
	\includegraphics[keepaspectratio, scale=0.4]{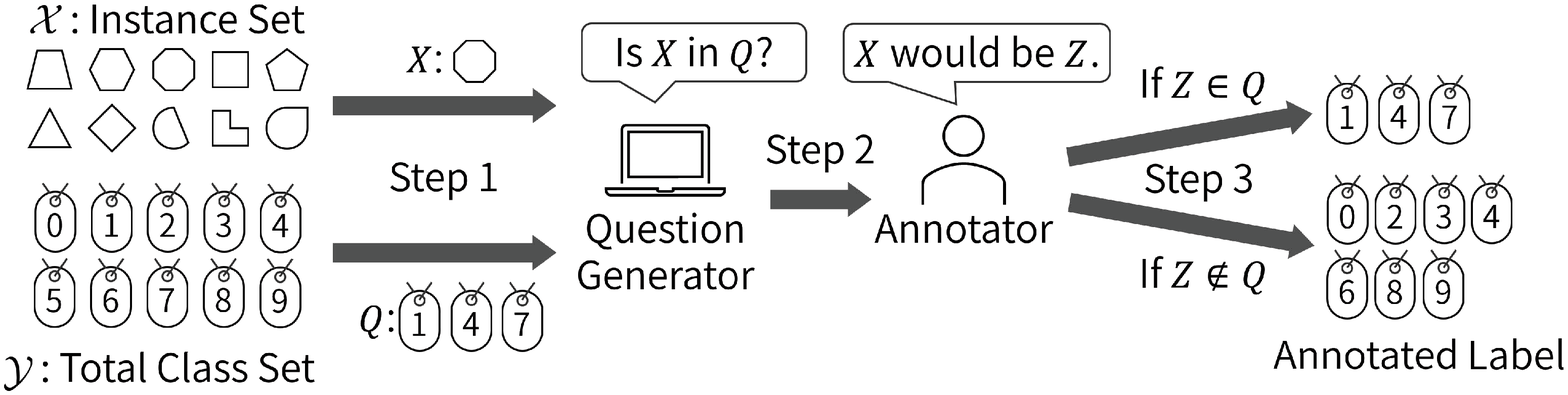}}
    \caption{Process of executing the Q$\&$A labeling.}\label{fig:error}
\end{figure*}

\begin{description}
\item[Step 1: Selecting Instance to Be Annotated and Question Class Set]\mbox{}\\
The question generator selects the instance $X$ to be annotated from instance set $\mathcal{X}$.
The question generator randomly selects $I$ classes from the total class set $\mathcal{Y}$.
The number of question items $I$ is assumed to be predetermined ($1\leq I \leq K-1$).
The set of selected classes is called the question class set, and its random variable is denoted by $Q~(\in {\mathfrak{B}_I(\mathcal{Y})})$.
\item[Step 2: Generating of Question and Presentation to Annotator]\mbox{}\\
The question generator generates a question about the selected instance $X$ and question class set $Q$ and presents the question to the annotator.
In this study, we considered the following two types of question generation methods.
\begin{description}
\item[Which-one-type:] Which one in $Q$ is $X$? 
\item[Is-in-type:] Is $X$ in $Q$?
\end{description}

\item[Step 3: Answering Questions and Labeling]\mbox{}\\
The annotator answers the question in different ways, depending on the question pattern.
The label $\widetilde{Y}$ assigned to the instance $X$ was determined based on the answer from the annotator.
Let $Z~(\in \mathcal{Y})$ be the class that the annotator naturally infers as its correct class when the instance $X$ is presented.
\begin{description}
\item[Which-one-type:]
The annotator answers ``$Z$'' if the class $Z$ is included in the question class set $Q$ ($Z \in Q$) and $\{Z\}~(\in \mathfrak{B}_1(\mathcal{Y}))$ is assigned to the instance $X$ as the label ($\widetilde{Y} = \{Z\}$).
If not included ($Z \notin Q$), the annotator answers ``not included'' and $\mathcal{Y}\backslash Q ~ (\in \mathfrak{B}_{K-I}(\mathcal{Y}))$ is assigned to the instance $X$ ($\tilde{y} = \mathcal{Y}\backslash Q$).
\item[Is-in-type:]
The annotator answers ``yes'' if the class $Z$ is included in the question class set $Q$ ($Z \in Q$) and $Q ~ (\in \mathcal{B}_I(\mathcal{Y}))$ is assigned to the instance $x$ as the label ($\widetilde{Y} = Q$).
If not included ($Z \notin Q$), the annotator answers ``no'' and $\mathcal{Y}\backslash Q ~ (\in \mathfrak{B}_{K-I}(\mathcal{Y}))$ is assigned to the instance $X$ ($\widetilde{Y} = \mathcal{Y}\backslash Q$).
\end{description}
\end{description}

\subsection{Which-one-type Q\&A Labeling}\label{sec:which-one-labeling}
In this section, we model annotations that follow the process of the which-one-type Q$\&$A labeling.
We denote the probability that a pair of instance $x$ and the which-one-type Q$\&$A label $\tilde{y}$ is generated as $\widetilde{P}_I(x,\tilde{y}):=\Pr\{X=x, \widetilde{Y}=\tilde{y}\}$.
Here, $X$ is a random variable representing the instance and $\widetilde{Y}$ is a random variable representing the which-one-type Q$\&$A label assigned to the instance.
We also assume that $\widetilde{P}_I(x,\tilde{y})$ can be decomposed into the product of $P(x):=\Pr\{X=x\}$ and $\widetilde{P}_I(\tilde{y}|x):=\Pr\{\widetilde{Y}=\tilde{y}|X=x\}$.
Then, the following theorem holds for the generative model of the which-one-type Q$\&$A label.
The proofs of the theorems, corollaries, and lemmas presented in this paper, including this theorem, are provided in the Appendix.

\begin{theorem}\label{thm:which-one-gen}
When assigning a label to an instance $x$ using the which-one-type Q$\&$A labeling, the probability that the assigned label is $\tilde{y}$ is given as follows:
\begin{align}
    \widetilde{P}_I(\tilde{y}|x) =
    \begin{cases}
       \displaystyle\frac{I}{K}\sum_{y \in \tilde{y}}P(y|x), & \text{if $\tilde{y} \in \mathfrak{B}_1(\mathcal{Y})$} \\
       \displaystyle\frac{1}{{}_K C_{I}}\sum_{y \in \tilde{y}}P(y|x), & \text{if $\tilde{y} \in \mathfrak{B}_{K-I}(\mathcal{Y})$} \\
    \end{cases}.
    \label{eq:generative_model_which_one}
\end{align}
\end{theorem}

Next, we discuss the information obtained from the instance assigned with the Q$\&$A label.
Based on the perspective provided in Sec.\ \ref{sec:recipients_perspective}, a receiver who receives an instance with the which-one-type Q$\&$A label containing multiple classes will assume the classes contained in the which-one-type Q$\&$A label as ground truth as uniform probability.
Let $\widehat{Y}$ be a random variable that represents the class assumed by the receiver of the data as the ground truth. Then, the following corollary holds for the information regarding the ground truth obtained by the receiver.

\begin{corollary}\label{cor:which-one-recv}
For an instance $x$ annotated by the which-one-type Q$\&$A labeling, the probability that the receiver of the instance assumes the ground truth label to be $\alpha~(\in \mathcal{Y})$ is given by
\begin{align}
    \Pr\{\widehat{Y}=\alpha|X=x\} = \beta P(\alpha|x) + (1-\beta)\frac{1}{K},
    \label{eq:inf_pt1}
\end{align}
where $\beta = I/(K-1)$
\end{corollary}

Comparing Eqs.\ \eqref{eq:privacy} and \eqref{eq:inf_pt1}, information regarding the label obtained by the receiver can be given in a similar form when receiving an instance with a candidate label and when receiving an instance with a Q$\&$A label.
When setting $N=K/(I+1)$ in the random noise $\beta$ of Eq.\ \eqref{eq:privacy}, it is consistent with Eq.\ \eqref{eq:inf_pt1}.
This implies that information regarding the label obtained by the receiver from the instance annotated by the which-one-type Q$\&$A labeling for the number of question items $I$ is equivalent to the information regarding the label obtained by the receiver from the instance assigned by candidate labeling with $N=K/(I+1)$.
In particular, the Q$\&$A label assigned by the which-one-type Q$\&$A labeling with the number of question items $I=K-1$ is equivalent to an ordinary label.
In the case of the which-one-type Q$\&$A labeling with the number of question items $I=K-1$, the annotator either answers the single class among the classes included in the question class set of size $K-1$, or answers ``not included'', i.e., the single class in the complement of the question class set.
Therefore, a single class is assigned as the label for all samples, and thus this method is equivalent to an ordinary labeling.

\subsection{Is-in-type Q$\&$A Labeling}
Here, we model annotations that follow the process of is-in-type
Q\&A labeling.
We denote the probability that a pair of instance $x$ and the is-in-type Q$\&$A label $\tilde{y}$ is generated as $\widetilde{P}_I(x,\tilde{y}):=\Pr\{X=x, \widetilde{Y}=\tilde{y}\}$.
Here, $X$ is a random variable representing the instance and $\widetilde{Y}$ is a random variable representing the is-in-type Q$\&$A label assigned to the instance.
We also assume that $\widetilde{P}_I(x,\tilde{y})$ can be decomposed into the product of $P(x):=\Pr\{X=x\}$ and $\widetilde{P}_I(\tilde{y}|x):=\Pr\{\widetilde{Y}=\tilde{y}|X=x\}$.
Then, the following theorem holds for the generative model of the is-in-type Q$\&$A label.

\begin{theorem}\label{thm:is-in-gen}
When assigning a label to an instance $x$ using the is-in-type Q$\&$A labeling, the probability that the assigned label is $\tilde{y}$ is given as follows.
\begin{align}
    \widetilde{P}_I(\tilde{y}|x) = 
       \frac{1}{\comb{K}{I}}\sum_{y \in \tilde{y}}P(y|x).
    \label{eq:generative_model_is_in}
\end{align}
\end{theorem}

Similar to the discussion presented in Sec.\ \ref{sec:which-one-labeling}, the following corollary holds for information about the ground truth that can be assumed by the receiver.

\begin{corollary}\label{cor:is-in-recv}
For an instance $x$ annotated by the is-in-type Q$\&$A labeling, the probability that the receiver of the instance assumes the ground-truth label to be $\alpha~(\in \mathcal{Y})$ is given by
\begin{align}
    \Pr\{\widehat{Y}=\alpha|X=x\} = \beta P(\alpha|x) + (1-\beta) \frac{1}{K},
    \label{eq:inf_pt2}
\end{align}
where $\beta = 1/(K-1)$.
\end{corollary}


Comparing Eqs.\ \eqref{eq:privacy} and \eqref{eq:inf_pt2}, the information about the label obtained by the receiver can be given in a similar form when receiving an instance with a candidate label and when receiving an instance with a Q$\&$A label.
When setting $N=K/2$ in the random noise $\beta$ of Eq.\ \eqref{eq:privacy}, it is consistent with Eq.\ \eqref{eq:inf_pt2}.
This implies that information regarding the label obtained by the receiver from the instance annotated by the is-in-type Q$\&$A labeling is equivalent to the information regarding the label obtained by the receiver from the instance assigned by candidate labeling with $N=K/2$, regardless of the number of question items $I$.



The results of Corollaries \ref{cor:which-one-recv} and \ref{cor:is-in-recv} show that, from the receiver's perspective, Q$\&$A labeling is partially equivalent to the candidate label.
They also demonstrated that the validity of Eq.\ \eqref{eq:CandP} assumed in \citep{ishida2017nips,katsura2020bridging} is guaranteed.
Note that, for the Q$\&$A label obtained by the which-one-type Q$\&$A labeling, the intensity of random noise $(1-\beta)$ increases as the number of question items $I$ increases. For the Q$\&$A label obtained by the is-in-type Q$\&$A labeling, the intensity of random noise $(1-\beta)$ is constant regardless of the number of question items $I$.
\section{Analysis}\label{sec:analysis}
\subsection{Which-one-type Q\&A Labeling}\label{sec:analysis-which-one}
To conduct learning based on empirical risk minimization (ERM) using instances with the which-one-type Q$\&$A label, a loss function should be specified to evaluate its performance.
The following theorem shows that the classification risk defined by Eq.\ \eqref{eq:pred_err}, as the expectation with respect to the generative model of the ordinary label, $P(X,Y)$, can be also expressed as the expectation with respect to the generative model of the Q$\&$A label, $\widetilde{P}_I(X,\widetilde{Y})$.

\begin{theorem}\label{thm:which-one-loss}
The classification risk defined in Eq.\ \eqref{eq:pred_err} is given by using the generative model of which-one-type Q$\&$A label, $\widetilde{P}_I(X,\widetilde{Y})$, as follows:
\begin{align}
    R(f) = \mathbb{E}_{\widetilde{P}_I(X, \widetilde{Y})}\Bigl[\widetilde{\mathcal{L}}_{I}(f(X), \widetilde{Y})\Bigl],
    \label{eq:pred_err_ord_which-one}
\end{align}
where
\begin{align}
    \widetilde{\mathcal{L}}_{I}(f(x), \tilde{y}) 
    &=\sum_{y \in \tilde{y}}\mathcal{L}(f(X), y) - \frac{(K-I)(K-I-1)}{I(2K-I-1)}\sum_{y \notin \tilde{y}}\mathcal{L}(f(X), y).
    \label{eq:loss-function_which-one}
\end{align}
\end{theorem}

This theorem shows that, if the loss function for learning by instances with the which-one-type Q$\&$A label is given by $\widetilde{\mathcal{L}}_{I}$, then its expectation with respect to $\widetilde{P}_I(X, \widetilde{Y})$ is consistent with the classification risk defined by Eq.\ \eqref{eq:pred_err}.
That is, this theorem defines a natural loss function for learning by instances with the which-one-type Q$\&$A labels. 
Note that, when the number of question items $I=K-1$, the loss function $\widetilde{\mathcal{L}}_{I}$ in the which-one-type Q$\&$A label learning is equivalent to the loss function $\mathcal{L}$ in ordinary label learning.

The Rademacher complexity of a set of real-valued functions $\mathcal{H}=\{h: \mathcal{X} \rightarrow \mathbb{R}\}$ is defined as follows:
\begin{align*}
    \mathfrak{R}_{n}(\mathcal{H})
    &= \mathbb{E}_{\mathcal{S}}\mathbb{E}_{\sigma}\left[
    \sup_{h\in\mathcal{H}}\frac{1}{n}
    \sum_{i=1}^{n}\sigma_i h(x_i)\right],
\end{align*}
where $S=\{x_1,...,x_n\}$ denotes the set of instances and $\sigma=\{\sigma_1,...,\sigma_n\}$ denotes the $n$ Rademacher variables that take $0$ or $1$ with equal probability independently.
Let $\widetilde{\mathcal{L}}_{I}$ given by Eq.\ \eqref{eq:pred_err_ord_which-one} be a Lipschitz continuous function and let the corresponding function set be
\begin{align*}
    \widetilde{\mathcal{H}}
    &=\{(x, \tilde{y})\mapsto\widetilde{\mathcal{L}}_{I}(f(x), \tilde{y})\mid
    f\in\mathcal{F}\}.
\end{align*}
Then, the following holds for the Rademacher complexity of $\widetilde{\mathcal{H}}$.

\begin{lemma}\label{lem:which_one_rade_loss}
    For the which-one-type Q$\&$A labeling, let $\rho$ be the Lipschitz coefficient of $\widetilde{\mathcal{L}}_{I}$, $g_y:\mathcal{X}\rightarrow \mathbb{R}$ be the binary discriminant function for classifying classes $y$ and others, and $\mathcal{G}_y=\{h: x\rightarrow g_y(x)\}$ be the set of real-valued functions corresponding to class $y$ over the instance space $\mathcal{X}$. Then, the Rademacher complexity of the function set $\widetilde{\mathcal{H}}$ satisfies
    \begin{align*}
        \mathfrak{R}_n(\widetilde{\mathcal{H}})\leq \frac{\sqrt{2}\rho K^2(K-1)}{I(2K-I-1)}\sum\limits_{y\in \mathcal{Y}} \mathfrak{R}_n(\mathcal{G}_y).
    \end{align*}
\end{lemma}

In general, because $\widetilde{P}_I(X,\widetilde{Y})$ is unknown, it is not possible to evaluate classification risk $R(f)$ based on the right side of Eq.\ \eqref{eq:pred_err_ord_which-one}.
As a practical alternative to minimizing the classification risk $R(f)$, we consider learning to minimize the empirical classification risk $\widehat{R}(f)$, which is defined as follows using $n$ samples, $\{(x_i, \tilde{y}_i)\}_{i=1}^n \thicksim \widetilde{P}_I(X,\widetilde{Y})$, obtained by the which-one-type Q$\&$A labeling.
\begin{align}
    \widehat{R}(f) = \frac{1}{n}\sum_{i=1}^n\widetilde{\mathcal{L}}_{I}(f(x_i), \tilde{y}_i).
    \label{eq:pt1_exp_err}
\end{align}
From Lemma \ref{lem:which_one_rade_loss}, the following holds for the hypothesis that minimizes $\widehat{R}(f)$, $\hat{f}:= \arg\min_{f\in\mathcal{F}}\widehat{R}(f)$ .
\begin{theorem}\label{thm:which_one_error}
    Let $C_L = \sup_{x\in\mathcal{X}, f\in\mathcal{F}}\mathcal{L}(f(x), Y)$ and $\rho$ be the Lipschitz coefficient of $\widetilde{\mathcal{L}}_{I}$.
    Then, the classification error $\mathcal{E}_{I}=R(\hat{f})-R(f^*)$ for the hypothesis $\hat{f}$ that minimizes $\widehat{R}(f)$ satisfies the following for any $\delta~(> 0)$ with a probability greater than $1 - \delta$.
    \begin{align}
    \mathcal{E}_{I} \leq
    &\frac{4\sqrt{2}\rho K^2(K-1)}{I(2K-I-1)}\sum\limits_{y\in \mathcal{Y}} \mathfrak{R}_n(\mathcal{G}_y) \nonumber\\
    &+ \frac{(K-I)(K^2+(I-2)K-I^2-1)}{I(2K-I-1)}C_{L}\sqrt{\frac{2\log{(2/\delta})}{n}}. \label{the:which_one_error}
    \end{align}
\end{theorem}
In Theorem \ref{thm:which_one_error}, the coefficients of the first and second terms on the right-hand side of Eq.\ \eqref{the:which_one_error} are monotonically decreasing with respect to $I$. 
Moreover, the statistical consistency can be guaranteed by the following arguments.
Let $\mathbb{H}$ be a reproducing kernel Hilbert space with a inner product $\langle\cdot,\cdot\rangle_{\mathbb{H}}$ defined over $\mathcal{X}$.
Let $k: \mathcal{X}\times\mathcal{X} \rightarrow \mathbb{R}$ be a corresponding positive definite symmetric kernel with $r^{2} = \sup_{x \in \mathcal{X}}k(x,x)$.
Additionally, let $\Phi: \mathcal{X} \rightarrow \mathbb{H}$ be a feature mapping associated with the reproducing kernel $k$.
Consider a linear-in-parameter model defined by
\begin{align*}
\mathcal{G}_{y} = \{g_{y}(x) = \langle\bm{w},\Phi(x)\rangle_{\mathbb{H}}: \|\bm{w}\|_{\mathbb{H}} \le \Lambda \}
\end{align*}
for some $\Lambda \ge 0$, where $\|\cdot\|_{\mathbb{H}}$ is a norm induced by $\langle\cdot,\cdot\rangle_{\mathbb{H}}$.
It is known that $\mathfrak{R}_n(\mathcal{G}_y) \le r\Lambda/\sqrt{n}$ \citep[Theorem 6.12]{mohri2012}, and thus when $n\rightarrow\infty$, we have $R(\hat{f})\rightarrow R(f^*)$, which shows that learning using the which-one-type Q$\&$A label has statistical consistency.

\subsection{Is-in-type Q\&A Labeling}
With the same discussion as in Sec.\ \ref{sec:analysis-which-one}, the following holds for learning based on ERM using instances with the is-in-type Q$\&$A label.
\begin{theorem}\label{thm:is-in-loss}
The classification risk defined in Eq.\ \eqref{eq:pred_err} is given by using the generative model of the is-in-type Q$\&$A label, $\widetilde{P}_I(X,\widetilde{Y})$, as follows:
\begin{align}
    R(f) = \mathbb{E}_{\widetilde{P}_I(X, \widetilde{Y})}\Bigl[\widetilde{\mathcal{L}}_{I}(f(X), \widetilde{Y})\Bigl],
    \label{eq:pred_err_ord_is-in}
\end{align}
where
\begin{align}
    \widetilde{\mathcal{L}}_{I}(f(x), \tilde{y})
    &=\sum_{y \in \tilde{y}}\mathcal{L}(f(x), y) - \frac{2I^2+K^2-K(2I+1)}{2I(K-I)}\sum_{y \notin \tilde{y}}\mathcal{L}(f(x), y).
    \label{eq:loss-function_is-in}
\end{align}
\end{theorem}

\begin{lemma}\label{lem:is_in_rade_loss}
    For the is-in-type Q$\&$A labeling, let $\rho$ be the Lipschitz coefficient of $\widetilde{\mathcal{L}}_{I}$, $g_y:\mathcal{X}\rightarrow \mathbb{R}$ be the binary discriminant function for classifying classes $y$ and others, and $\mathcal{G}_y=\{h:x\rightarrow g_y(x)\}$ be the set of real-valued functions corresponding to class $y$ over the instance space $\mathcal{X}$. Then, the Rademacher complexity of the function set $\widetilde{\mathcal{H}}$ satisfies
    \begin{align*}
        \mathfrak{R}_n(\widetilde{\mathcal{H}})\leq \frac{\sqrt{2}\rho K^2(K-1)}{2I(K-I)}\sum\limits_{y\in \mathcal{Y}} \mathfrak{R}_n(\mathcal{G}_y).
    \end{align*}
\end{lemma}
\begin{theorem}\label{thm:is_in_error}
    Let $C_L = \sup_{x\in\mathcal{X}, f\in\mathcal{F}}\mathcal{L}(f(x), Y)$ and $\rho$ be the Lipschitz coefficient of $\widetilde{\mathcal{L}}_{I}$.
    Then, the classification error $\mathcal{E}_{I}=R(\hat{f})-R(f^*)$ for the hypothesis $\hat{f}$ that minimizes $\widehat{R}(f)$ satisfies the following for any $\delta~(> 0)$, with a probability greater than $1 - \delta$.
    \begin{align}
    \mathcal{E}_{I} \leq 
        \frac{\sqrt{2}\rho K^2(K-1)}{I(K-I)}\sum\limits_{y\in \mathcal{Y}} \mathfrak{R}_n(\mathcal{G}_y) + \frac{K(K-1)}{2\min{(I, K-I)}}C_{L}\sqrt{\frac{2\log{(2/\delta})}{n}} 
       \label{the:is_in_error}
    \end{align}
\end{theorem}
Note that when the upper bound in Theorem \ref{thm:is_in_error} is treated as a function of $I$, it takes its minimum value when $I=K/2$. 
In addition, learning using the is-in-type Q$\&$A label has statistical consistency as with learning using the which-one-type Q$\&$A label.
\section{Experiment}\label{sec:experiment}
In Sec.\ \ref{sec:analysis}, we theoretically evaluated the upper bound of classification errors of classifiers trained on instances labeled with Q\&A labels as training data.
The upper bound of classification errors monotonically decreases with the number of question items in the which-one-type, and takes smaller values as the number of question items approaches $K/2$ in the is-in-type.
However, we can only theoretically evaluate the upper bound of classification errors and cannot directly evaluate the classification error.
Therefore, we quantitatively evaluate the classification error through experiments.

\subsection{Generation of Datasets}\label{sec:experiment_pre}
In the experiment, we used the MNIST\footnote{\url{http://yann.lecun.com/exdb/mnist/}}, Kuzushiji-MNIST\footnote{\url{https://github.com/rois-codh/kmnist}}, and Fashion-MNIST\footnote{\url{https://github.com/zalandoresearch/fashion-mnist}} datasets.
The summary of the datasets is presented in Table \ref{tab:dataset}.
We applied the two Q\&A labeling procedures described in Sec.\  \ref{sec:analysis} to assign Q\&A labels to each instance in the three datasets.
First, the question generator randomly selects classes from the total class set according to the predetermined number of question items to generate the question class set.
Here, the ordinary labels already assigned to these datasets used in the experiment are treated as the class that the annotator considers most appropriate for each instance.
Next, the Q\&A label is determined by whether or not its ordinary label is included in the question class set.
We perform this process for all possible numbers of question items in the two types of Q\&A labeling.
With this process, the Q\&A label $\widetilde{y}$ assigned to an instance $x$ will follow the generative model defined by Eq.\  \eqref{eq:generative_model_which_one} or Eq.\  \eqref{eq:generative_model_is_in}.

\begin{table}[t]\centering
\caption{
Datasets used in the experiment.
}\scalebox{1.0}{
  \begin{tabular}{c|c c c c} \hline
     & MNIST & Kuzushiji-MNIST & Fashion-MNIST \\ \hline
     Number of classes & 10 & 10 & 10 \\
     Number of dimensions & 28$\times$28 & 28$\times$28 & 28$\times$28 \\
     Color mode & gray & gray & RGB \\
     Number of data & 60,000 & 60,000 & 60,000 \\ \hline
  \end{tabular}
  }\label{tab:dataset}
\end{table}

\begin{figure*}[t]\centering
    \subfigure[which-one-type]{
		\includegraphics[width=1.00\linewidth]{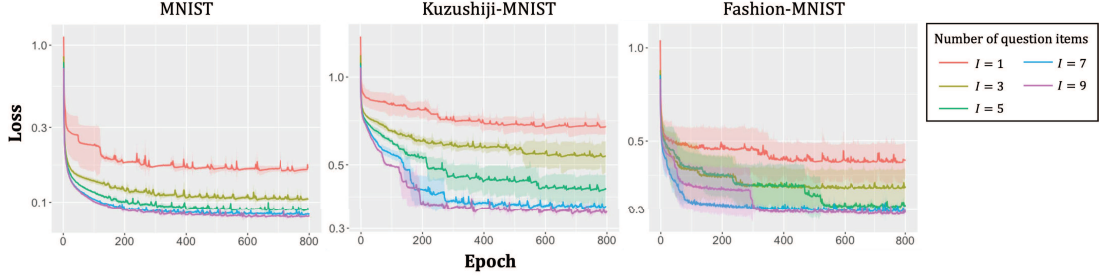}}
    \subfigure[is-in-type]{
		\includegraphics[width=1.00\linewidth]{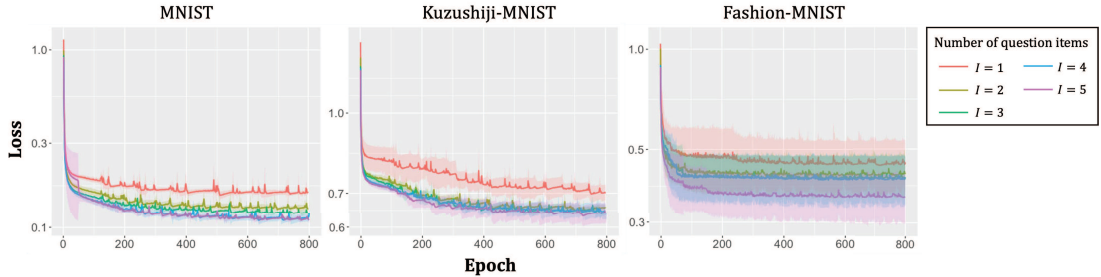}}
    \caption{The average and standard deviation of the classification accuracy on the test data for a 10-class classifier trained on instances labeled with different types of Q\&A labeling and a varying number of question items $I$. The colors indicate the number of question items in the Q\&A labeling used to label each instance in the training data. The vertical axis is shown on a logarithmic scale.}
\label{fig:learning}
\end{figure*}

\subsection{Experimental Settings}
We compare the classification errors of classifiers trained using training data labeled with Q\&A labels by Q\&A labeling for each instance of three datasets: MNIST, Kuzushiji-MNIST, and Fashion-MNIST.
Each dataset contains 10 classes.
The comparison was conducted by varying the number of question items used to obtain the Q\&A labels.
Assuming the existence of a classifier $f^{*} \in \mathcal{F}$ that can classify all instances as their ground truth class, it holds that $R(f^{*})=0$.
Under this assumption, the classification error $\mathcal{E}_{I}=R(\hat{f})-R(f^{*})$ is equivalent to the predictive discrimination error $R(\hat{f})$ in Eq.\ \eqref{eq:pred_err}.
In practice, the predictive classification risk $R(\hat{f})$ cannot be obtained, thus instead, we evaluate the empirical classification risk $\widehat{R}(\hat{f})$ of the test data where the ordinary labels are assigned.
Note that the empirical classification risk $\widehat{R}(f)$ of classifier $f$ is defined using $n$ test samples, ${(x_i, y_i)}_{i=1}^n \thicksim P(X,Y)$, as follows.
\begin{align}
    \widehat{R}(f) = \frac{1}{n}\sum_{i=1}^n\mathcal{L}(f(x_i), y_i).
    \label{eq:exp_err}
\end{align}

Using the ordinary labels pre-assigned to the dataset, we extract 1,000 instances for each class and annotate them with Q\&A labels using the method in Sec.\ \ref{sec:experiment_pre} as the training data.
Each instance is randomly selected from the original datasets.
The classifier is a neural network with one hidden layer containing 500 nodes.
The activation function used for the hidden layer is the relu function, and the softmax function is used for the output layer.
The batch size is set to 500 during training.
Equations \eqref{eq:loss-function_which-one} and \eqref{eq:loss-function_is-in} are used as the loss function, and MAE (Mean Absolute Error) is used as the function $\mathcal{L}$.
The function $\mathcal{L}$ is also used to calculate the empirical classification error $\widehat{R}(\hat{f})$ in Eq.\ \eqref{eq:exp_err}.
MAE is defined as follows:
\begin{align*}
    \mathcal{L}(f(x), y) = \sum_{k=1}^{K}\mid f(x)_k - y_k\mid
\end{align*}
The classifier is $f:\mathcal{X}\rightarrow \mathbb{R}^K$, and $y\in \mathbb{R}^K$ is a one-hot vector that takes the value of 1 at its corresponding position and 0 elsewhere.
Each subscript represents the position of the vector.

The optimization algorithm used was Adam \citep{kingma2015}, where the number of times to repeatedly train on a single training data instance (i.e., epoch) was set to 800, weight decay was set to $10^{-3}$, and the learning rate was set to $10^{-2}$.
The process of assigning Q\&A labels and training the classifier is performed $5$ times.

\subsection{Experimental Results}
The experimental results are shown in Fig.\ \ref{fig:learning}.
For the which-one-type, the results are displayed for $I=1,3,5,7,9$.
For the is-in-type, the results are displayed for $I=1,2,3,4,5$, as the performance of the classifier is equivalent for $I=J$ and $I=10-J$.

For the classifier trained on instances labeled with which-one-type Q\&A labeling, it shows that the empirical discriminant error on the test data decreases as the number of question items increases. 
For the classifier trained on instances labeled with is-in-type Q\&A labeling, it shows that the empirical discriminant error on the test data decreases as the number of question items increases when $I<5~(=K/2)$.
Therefore, it can be concluded that the classification error of the classifier trained using Q\&A labels exhibits the same behavior as the upper bound of the classification error that was theoretically shown in Sec.\ \ref{sec:analysis}.

\section{Conclusion}\label{sec:conclusion}
In this paper, we proposed Q$\&$A labeling, an annotation method that is feasible even when ordinal labels are difficult to assign.
The proposed Q$\&$A labeling is executed by a question generator that asks questions about the labels to be assigned and an annotator that answers the questions and makes annotations.
The question generator asks a question by presenting the annotator with the question class set, which is a subset of the total class set.
Depending on how the question was asked, the annotator answers the question by considering only the classes included in the question class set.
In this paper, we proposed the which-one-type Q$\&$A labeling and the is-in-type Q$\&$A labeling, which differ in the manner in which questions are asked and answered.
We then derived a generative model for the labels assigned by these two types of Q$\&$A labeling, respectively.
We also showed that the derived generative model is partially equivalent to that assumed in existing studies as a generative model for the candidate label, which is a generalization of the complementary label.
This implies that the proposed annotation method provides a specific procedure for realizing the generative model assumed in previous studies and guarantees the validity of the theoretical analysis performed in existing studies based on this assumption.
In addition, we derived a loss function to evaluate the classification risk and upper bound of the classification error using the Q$\&$A labels.
The derived classification errors indicate that there is statistical consistency in learning using the Q$\&$A labels.
Experimental also showed that the classification error itself has the same behavior as the upper bound of the classification error.
Future work will include the optimization of question/answer methods and the extraction of question class sets.
\section*{Appendix A: Which-one-type Q\&A Labeling}\label{sec:calculate1}
\subsection*{Proof of Theorem \ref{thm:which-one-gen}}
\begin{lemma}\label{lem:which-one-y-given-x-and-z}
When the class assumed by the annotator for instance $x$ is $z$, the generative model of the which-one-type Q$\&$A label for that instance is given as follows:
\begin{align}
\Pr\{Y=\{y\}|X=x, Z=z\} &=
\begin{cases}
\displaystyle\frac{I}{K}, & \text{if $z = y$} \\
0, & \text{if $z \neq y$} \\
\end{cases},
\quad \forall y \in \mathcal{Y},
\label{eq:which-one-gen-1}
\\
\Pr\{Y=\mathcal{Y}\backslash q|X=x, Z=z\} &=
\begin{cases}
0, & \text{if $z \in q$} \\
\displaystyle\frac{1}{{}_K C_I}, & \text{if $z \not\in q$} \\
\end{cases},
\quad \forall q \in \mathfrak{B}_{I}(\mathcal{Y}).
\label{eq:which-one-gen-2}
\end{align}
\end{lemma}
\begin{proof}
First, we derive Eq.\ \eqref{eq:which-one-gen-1}.
This is trivial when $z \neq y$.
For $z = y$, the following holds:
\begin{align*}
\Pr\{Y=\{y\}|X=x, Z=z\}
&= \sum_{q \in \mathfrak{B}_I(\mathcal{Y})}\Pr\{y \in q|X=x, Z = z, Q=q\}\Pr\{Q=q\}\\
&= {}_K C_I \frac{I}{K}\frac{1}{{}_K C_I}\\
&= \frac{I}{K},
\end{align*}
where we used that $\Pr\{y \in q|X=x, Z = z, Q=q\} = I/K$ when $z = y$.

Next, we derive Eq.\ \eqref{eq:which-one-gen-2}.
This is trivial for $z \in q$.
For $z \in q$, the following holds.
\begin{align*}
\Pr\{Y=\mathcal{Y}\backslash q|X=x, Z=z\}
&= \Pr\{z \notin q|X=x, Z=z, Q=q\} \Pr\{Q=q\}\\
&= \frac{1}{{}_K C_I},
\end{align*}
where we used that $\Pr\{z \notin q|X=x, Z=z, Q=q\} = 1$ when $z \not\in q$.
\end{proof}

Using Lemma \ref{lem:which-one-y-given-x-and-z}, we can prove Theorem \ref{thm:which-one-gen}.
First, for the instance, $x$, the probability that label $\{y\}~(\in \mathfrak{B}_1(\mathcal{Y}))$ can be derived as follows:
\begin{align*}
    \Pr(Y=\{y\}|x)
    &= \sum_{z\in \mathcal{Y}}\Pr\{Y=\{y\}|X=x,Z=z\}\Pr\{Z=z|X=x\}\nonumber\\
    &= \frac{I}{K}\Pr\{Z=y|X=x\}.
\end{align*}
In addition, the probability that the label $\mathcal{Y}\backslash q (\in\mathfrak{B}_{K-I}~(\mathcal{Y}))$ is assigned to the instance $x$ can be derived as follows.
\begin{align*}
    \Pr\{Y=\mathcal{Y}\backslash q|X=x\}
    &= \sum_{z \in \mathcal{Y}}\Pr\{Y=\mathcal{Y}\backslash q|X=x, Z=z\}\Pr\{Z=z|X=x\}\nonumber\\
    &= \frac{1}{\comb{K}{I}}\sum_{z \in \mathcal{Y}\backslash q}\Pr\{Z=z|X=x\}.
\end{align*}

\subsection*{Proof of Corollary \ref{cor:which-one-recv}}
First, the following holds:
\begin{align}
    &\Pr\{\widehat{Y}=\alpha|X=x\}\nonumber\\
    &= \Pr\{\widehat{Y} =\alpha|X=x, \alpha \in Y \in \mathfrak{B}_1(\mathcal{Y})\}\Pr\{\alpha \in Y \in \mathfrak{B}_1(\mathcal{Y})|X=x\} \nonumber \\
    &\hspace{2em}+ \Pr\{\widehat{Y} =\alpha|X=x, \alpha \notin Y \in \mathfrak{B}_1(\mathcal{Y})\}\Pr\{\alpha \notin Y \in \mathfrak{B}_1(\mathcal{Y})|X=x\} \nonumber \\
    &\hspace{2em}+ \Pr\{\widehat{Y} =\alpha|X=x, \alpha \in Y \in \mathfrak{B}_{K-I}(\mathcal{Y})\}\Pr\{\alpha \in Y \in \mathfrak{B}_{K-I}(\mathcal{Y})|X=x\} \nonumber \\
    &\hspace{2em}+ \Pr\{\widehat{Y} =\alpha|X=x, \alpha \notin Y \in \mathfrak{B}_{K-I}(\mathcal{Y})\}\Pr\{\alpha \notin Y \in \mathfrak{B}_{K-I}(\mathcal{Y})|X=x\} \nonumber \\
    &= \Pr\{\widehat{Y} =\alpha|X=x, \alpha \in Y \in \mathfrak{B}_1(\mathcal{Y})\}\Pr\{\alpha \in Y \in \mathfrak{B}_1(\mathcal{Y})|X=x\} \nonumber \\
    &\hspace{2em}+ \Pr\{\widehat{Y} =\alpha|X=x, \alpha \in Y \in \mathfrak{B}_{K-I}(\mathcal{Y})\}\Pr\{\alpha \in Y \in \mathfrak{B}_{K-I}(\mathcal{Y})|X=x\} \nonumber \\
    &= \Pr\{\alpha \in Y \in \mathfrak{B}_1(\mathcal{Y})|X=x\} + \frac{1}{K-I}\Pr\{\alpha \in Y \in \mathfrak{B}_{K-I}(\mathcal{Y})|X=x\}. \label{eq:which-one-recv-1}
\end{align}
In the third equality, we used $\Pr\{\widehat{Y} =\alpha|X=x, \alpha \notin Y \in \mathfrak{B}_I(\mathcal{Y})\} = 0$, $\Pr\{\widehat{Y} =\alpha|X=x, \alpha \notin Y \in \mathfrak{B}_{K-I}(\mathcal{Y})\} = 0$.
In the last equality, we used $\Pr\{\widehat{Y} =\alpha|X=x, \alpha \in Y \in \mathfrak{B}_1(\mathcal{Y})\}=1$, $\Pr\{\widehat{Y} =\alpha|X=x, \alpha \in Y \in \mathfrak{B}_{K-I}(\mathcal{Y})\}=1/(K-I)$.

For the first term on the right-hand side of Eq.\ \eqref{eq:which-one-recv-1}, the following holds.
\begin{align}
    \Pr\{\alpha \in Y \in \mathfrak{B}_1(\mathcal{Y})|X=x\}
    &= \sum_{\tilde{y} \in \mathfrak{B}_1(\mathcal{Y})}\mathbbm{1}(\alpha \in \tilde{y})\widetilde{P}_I(\tilde{y}|x) \nonumber \\
    &= \frac{I}{K}\sum_{\tilde{y} \in \mathfrak{B}_1(\mathcal{Y})}\mathbbm{1}(\alpha \in \tilde{y})\sum_{y \in \tilde{y}}P(y|x) \nonumber \\
    &= \frac{I}{K}P(\alpha|x).
    \label{eq:which-one-recv-2}
\end{align}
In the second equality, we use Theorem \ref{thm:which-one-gen}.
In addition, for the second term on the right-hand side of Eq.\ \eqref{eq:which-one-recv-1}, the following holds:
\begin{align}
    \Pr\{\alpha \in Y \in \mathfrak{B}_{K-I}(\mathcal{Y})|X=x\}
    &= \sum_{\tilde{y} \in \mathfrak{B}_{K-I}(\mathcal{Y})}\mathbbm{1}(\alpha \in \tilde{y})\widetilde{P}_I(\tilde{y}|x) \nonumber \\
    &= \frac{1}{{}_K C_I}\sum_{\tilde{y} \in \mathfrak{B}_{K-I}(\mathcal{Y})}\mathbbm{1}(\alpha \in \tilde{y})\sum_{y \in \tilde{y}}P(y|x) \nonumber \\
    &= \frac{1}{{}_K C_I}\{{}_{K-2} C_{K-I-1} P(\alpha|x) + {}_{K-2} C_{K-I-2}\} \nonumber \\
    &= \frac{I(K-I)}{K(K-1)}P(\alpha|x) + \frac{(K-I)(K-I-1)}{K(K-1)}.
    \label{eq:which-one-recv-3}
\end{align}
In the second equality, we used Theorem \ref{thm:which-one-gen}.
In the third equality, the fact that the following holds for $2 \le J \le K$ is used.
\begin{align}
    \sum_{\tilde{y} \in \mathfrak{B}_{J}(\mathcal{Y})}\mathbbm{1}(\alpha \in \tilde{y})\sum_{y \in \tilde{y}}P(y|x)
    &= {}_{K-1} C_{J-1}P(\alpha|x) + {}_{K-2} C_{J-2}\sum_{y \neq \alpha}P(y|x)\nonumber\\
    &= {}_{K-2} C_{J-1}P(\alpha|x) + {}_{K-2} C_{J-2}.
    \label{eq:which-one-recv-4}
\end{align}

Substituting Eqs.\ \eqref{eq:which-one-recv-2} and \eqref{eq:which-one-recv-3} into Eq.\ \eqref{eq:which-one-recv-1}, Corollary \ref{cor:which-one-recv} holds as follows.
\begin{align*}
    \Pr\{\widehat{Y}=\alpha|X=x\} &= \frac{I}{K}P(\alpha|x) + \frac{I}{K(K-1)}P(\alpha|x) + \frac{(K-I-1)}{K(K-1)} \nonumber \\
    &=\frac{I}{K-1}P(\alpha|x) + \frac{K-I-1}{K-1} \cdot \frac{1}{K}.
\end{align*}

\subsection*{Proof of Theorem \ref{thm:which-one-loss}}
\begin{lemma}\label{lem:which-one-pdf}
The probability that an ordinary label $y$ is assigned to an instance $x$, $P(y|x)$, can be expressed using the probability of a which-one-type $Q\&A$ Label $\tilde{y}$ being assigned for the given instance $x$ as follows:
\begin{align}
    P(y|x) = \frac{K(K-1)}{I(2K-I-1)}
    \sum_{y \in \tilde{y} \in \mathfrak{B}_1(\mathcal{Y}) \cup \mathfrak{B}_{K-I}(\mathcal{Y})} \widetilde{P}_I(\tilde{y}|x) - \frac{(K-I)(K-I-1)}{I(2K-I-1)}. \label{eq:which-one-pdf}
\end{align}
\end{lemma}

\begin{proof}
First, the following holds.
\begin{align}
    &\Pr\{y \in \widetilde{Y} \in \mathfrak{B}_1(\mathcal{Y}) \cup \mathfrak{B}_{K-I}(\mathcal{Y})|X=x\}\nonumber\\
    &\hspace{3em}= \sum_{y \in \tilde{y} \in \mathfrak{B}_1(\mathcal{Y}) \cup \mathfrak{B}_{K-I}(\mathcal{Y})} \widetilde{P}_I(\tilde{y}|x)
    = \sum_{y \in \tilde{y} \in \mathfrak{B}_1(\mathcal{Y})}\tilde{P}_I(\tilde{y}|x) + \sum_{y \in \tilde{y} \in \mathfrak{B}_1(\mathcal{Y})} \tilde{P}_I(\tilde{y}|x).
    \label{eq:which-one-pdf-proof}
\end{align}

According to Theorem \ref{thm:which-one-gen}, for the first term on the right-hand side of Eq.\ \eqref{eq:which-one-pdf-proof}, the following holds:
\begin{align*}
    \sum_{y \in \tilde{y} \in \mathfrak{B}_1(\mathcal{Y})}\tilde{P}_I(\tilde{y}|x)
    = \tilde{P}_I(\{y\}|x) 
    = \frac{I}{K}P(y|x).
\end{align*}

Similarly, according to Theorem \ref{thm:which-one-gen}, for the second term on the right-hand side of Eq.\ \eqref{eq:which-one-pdf-proof}, the following holds:
\begin{align*}
    \sum_{y \in \tilde{y} \in \mathfrak{B}_{K-I}(\mathcal{Y})}\widetilde{P}_I(\tilde{y}|x)
    &= \frac{1}{\comb{K}{I}}\sum_{y \in \tilde{y} \in \mathfrak{B}_{K-I}(\mathcal{Y})}\sum_{y' \in \tilde{y}}P(y'|x) \\
    &= \frac{1}{\comb{K}{I}}\sum_{\tilde{y} \in \mathfrak{B}_{K-I}}\mathbbm{1}(y \in \tilde{y})\sum_{y' \in \tilde{y}}P(y'|x)\\
    &= \frac{1}{\comb{K}{I}}\{\comb{K-2}{K-I-2}P(y|x) + \comb{K-2}{K-I-2}\} \\
    &= \frac{I(K-I)}{K(K-1)}P(y|x) + \frac{(K-I)(K-I-1)}{K(K-1)}.
\end{align*}
For the third equality, we used Eq.\ \eqref{eq:which-one-recv-4}.

From the above, the following holds
\begin{align*}
    \sum_{y \in \tilde{y} \in \mathfrak{B}_1(\mathcal{Y}) \cup \mathfrak{B}_{K-I}(\mathcal{Y})} \widetilde{P}_I(\tilde{y}|x)
    &= \frac{I}{K}P(y|x) + \frac{I(K-I)}{K(K-1)}P(y|x) + \frac{(K-I)(K-I-1)}{K(K-1)} \\
    &= \frac{I(2K-I-1)}{K(K-1)}P(y|x) + \frac{(K-I)(K-I-1)}{K(K-1)}.
\end{align*}
By rearranging this result, Eq.\ \eqref{eq:which-one-pdf} is derived.
\end{proof}

For the classification risk defined in Eq.\ \eqref{eq:pred_err}, the following holds:
\begin{align}
    R(f) = \mathbb{E}_{P(X)}\Bigl[\mathbb{E}_{P(Y|X)}\Bigl[\mathcal{L}(f(x), y)\Bigl]\Bigl]. \label{eq:which-one_expected}
\end{align}
Using Lemma \ref{lem:which-one-pdf}, the following holds:
\begin{align*}
    &\mathbb{E}_{P(Y|x)}[\mathcal{L}(f(x), Y)]\\
    &= \sum_{y \in \mathcal{Y}}\mathcal{L}(f(x), y)P(y|x)\\
    &= \sum_{y \in \mathcal{Y}}\mathcal{L}(f(x), y)\left\{\frac{K(K-1)}{I(2K-I-1)}
    \sum_{y \in \tilde{y} \in \mathfrak{B}_1(\mathcal{Y}) \cup \mathfrak{B}_{K-I}(\mathcal{Y})} \widetilde{P}_I(\tilde{y}|x) - \frac{(K-I)(K-I-1)}{I(2K-I-1)}\right\}\\
    &= \frac{K(K-1)}{I(2K-I-1)}\sum_{y \in \mathcal{Y}}\sum_{y\in\tilde{y}\in \mathfrak{B}_1(\mathcal{Y}) \cup \mathfrak{B}_{K-I}(\mathcal{Y})}\mathcal{L}(f(x), y)\widetilde{P}_I(\tilde{y}|x) - \frac{(K-I)(K-I-1)}{I(2K-I-1)}\sum_{y \in \mathcal{Y}}\mathcal{L}(f(x), y)\\
    &= \frac{K(K-1)}{I(2K-I-1)}\sum_{\tilde{y}\in \mathfrak{B}_1(\mathcal{Y}) \cup \mathfrak{B}_{K-I}(\mathcal{Y})}\sum_{y \in \tilde{y}}\mathcal{L}(f(x), y)\widetilde{P}_I(\tilde{y}|x) - \frac{(K-I)(K-I-1)}{I(2K-I-1)}\sum_{y \in \mathcal{Y}}\mathcal{L}(f(x), y)\\
    &= \frac{K(K-1)}{I(2K-I-1)}\sum_{\tilde{y}\in \mathfrak{B}_1(\mathcal{Y}) \cup \mathfrak{B}_{K-I}(\mathcal{Y})}\widetilde{P}_I(\tilde{y}|x)\sum_{y \in \tilde{y}}\mathcal{L}(f(x), y) - \frac{(K-I)(K-I-1)}{I(2K-I-1)}\sum_{y \in \mathcal{Y}}\mathcal{L}(f(x), y)\\
    &= \mathbb{E}_{\widetilde{P}_I(\widetilde{Y}|x)}\Biggl[\frac{K(K-1)}{I(2K-I-1)}\sum_{y \in \widetilde{Y}}\mathcal{L}(f(x), y) - \frac{(K-I)(K-I-1)}{I(2K-I-1)}\sum_{y \in \mathcal{Y}}\mathcal{L}(f(x), y)\Biggr]\\
    &= \mathbb{E}_{\widetilde{P}_I(\widetilde{Y}|x)}\Biggl[\frac{K(K-1)}{I(2K-I-1)}\sum_{y \in \widetilde{Y}}\mathcal{L}(f(x), y) \nonumber\\
    &\hspace{5em} - \frac{(K-I)(K-I-1)}{I(2K-I-1)}\sum_{y \in \widetilde{Y}}\mathcal{L}(f(x), y) - \frac{(K-I)(K-I-1)}{I(2K-I-1)}\sum_{y \notin \widetilde{Y}}\mathcal{L}(f(x), y)\Biggr]\\
    &= \mathbb{E}_{\widetilde{P}_I(\widetilde{Y}|x)}\left[\sum_{y \in \widetilde{Y}}\mathcal{L}(f(x), y) - \frac{(K-I)(K-I-1)}{I(2K-I-1)}\sum_{y \notin \widetilde{Y}}\mathcal{L}(f(x), y)\right].
\end{align*}
Therefore, 
\begin{align*}
    \widetilde{\mathcal{L}}_{I}(f(X), \widetilde{Y}) &=
       \sum_{y \in \widetilde{Y}}\mathcal{L}(f(x), y) - \frac{(K-I)(K-I-1)}{I(2K-I-1)}\sum_{y \notin \widetilde{Y}}\mathcal{L}(f(x), y).
\end{align*}
Subsequently, by substituting this into Eq.\ \eqref{eq:which-one_expected}, the following holds:
\begin{align*}
    R(f) = \mathbb{E}_{\widetilde{P}_I(X, \widetilde{Y})}\Bigl[\widetilde{\mathcal{L}}_{I}(f(X), \widetilde{Y})\Bigl].
\end{align*}

\subsection*{Proof of Lemma \ref{lem:which_one_rade_loss}}
If $\alpha_{i} = 2\mathbbm{1}(y \in \tilde{y}_{i})-1$, for the Rademacher complexity $\mathfrak{R}_n(\widetilde{\mathcal{H}})$ of the function set $\widetilde{\mathcal{H}}$, the following holds:
\begin{align}
    &\mathfrak{R}_n(\widetilde{\mathcal{H}})\nonumber\\
    &= \mathbb{E}_{\widetilde{P}_I(X, \widetilde{Y})}\mathbb{E}_{\sigma}\Biggl[\sup_{h\in\widetilde{\mathcal{H}}}\frac{1}{n}\sum_{i=1}^n\sigma_i h(x_i, \tilde{y}_i)\Biggl] \nonumber \\
    &= \mathbb{E}_{\widetilde{P}_I(X, \widetilde{Y})}\mathbb{E}_{\sigma}\Biggl[\sup_{f\in\mathcal{F}}\frac{1}{n}\sum_{i=1}^n \sigma_i \Biggl\{\sum_{y\in \tilde{y}_i}\mathcal{L}(f(x_i), y) - \frac{(K-I)(K-I-1)}{I(2K-I-1)}\sum_{y \notin \tilde{y}_i}\mathcal{L}(f(x_i), y)\Biggl\}\Biggl] \nonumber \\
    &\le \mathbb{E}_{\widetilde{P}_I(X, \widetilde{Y})}\mathbb{E}_{\sigma}\Biggl[\sup_{f\in\mathcal{F}}\frac{1}{n}\sum_{i=1}^n \sigma_i \sum_{y\in \tilde{y}_i}\mathcal{L}(f(x_i), y)\Biggr] \nonumber \\ 
    &\hspace{1em}+\mathbb{E}_{\widetilde{P}_I(X, \widetilde{Y})}\mathbb{E}_{\sigma}\Biggl[\sup_{f\in \mathcal{F}} \frac{1}{n}\sum_{i=1}^n\sigma_i\frac{(K-I)(K-I-1)}{I(2K-I-1)} \sum_{y\notin \tilde{y}_i}\mathcal{L}(f(x_i), y)\Biggr] \nonumber \\
    &= \mathbb{E}_{\widetilde{P}_I(X, \widetilde{Y})}\mathbb{E}_{\sigma}\Biggl[\sup_{f\in\mathcal{F}}\frac{1}{2n}\sum_{i=1}^n \sigma_i \sum_{y\in \mathcal{Y}}\mathcal{L}(f(x_i), y)(\alpha_i + 1)\Biggr] \nonumber \\ 
    &\hspace{1em}+\mathbb{E}_{\widetilde{P}_I(X, \widetilde{Y})}\mathbb{E}_{\sigma}\Biggl[\sup_{f\in \mathcal{F}} \frac{1}{2n}\sum_{i=1}^n\sigma_i\frac{(K-I)(K-I-1)}{I(2K-I-1)}\sum_{y \in \mathcal{Y}}\mathcal{L}(f(x_i), y)(1-\alpha_i)\Biggr] \nonumber \\
    &= \mathbb{E}_{\widetilde{P}_I(X, \widetilde{Y})}\mathbb{E}_{\sigma}\Biggl[\sup_{f\in\mathcal{F}}\frac{1}{2n}\sum_{i=1}^n \biggl(\sigma_i \sum_{y\in \mathcal{Y}}\alpha_i\mathcal{L}(f(x_i), y) + \sigma_i \sum_{y\in \mathcal{Y}}\mathcal{L}(f(x_i), y)
    \biggl)\Biggr] \nonumber \\ 
    &\hspace{1em}+\frac{(K-I)(K-I-1)}{I(2K-I-1)}\mathbb{E}_{\widetilde{P}_I(X, \widetilde{Y})}\mathbb{E}_{\sigma}\Biggl[\sup_{f\in \mathcal{F}} \frac{1}{2n}\sum_{i=1}^n\biggl(\sigma_i\sum_{y \in \mathcal{Y}}\mathcal{L}(f(x_i), y) - \sigma_i\sum_{y \in \mathcal{Y}}\alpha_i\mathcal{L}(f(x_i), y)\biggl)\Biggr] \nonumber \\
    &\le \mathbb{E}_{\widetilde{P}_I(X, \widetilde{Y})}\mathbb{E}_{\sigma}\Biggl[\sup_{f\in\mathcal{F}}\frac{1}{n}\sum_{i=1}^n \sigma_i \sum_{y\in \mathcal{Y}}\mathcal{L}(f(x_i), y)
    \Biggr] \nonumber \\
    &\hspace{1em}+\frac{(K-I)(K-I-1)}{I(2K-I-1)}\mathbb{E}_{\widetilde{P}_I(X, \widetilde{Y})}\mathbb{E}_{\sigma}\Biggl[\sup_{f\in \mathcal{F}} \frac{1}{n}\sum_{i=1}^n\sigma_i\sum_{y \in \mathcal{Y}}\mathcal{L}(f(x_i), y)\Biggr] \nonumber \\
    &= \frac{K(K-1)}{I(2K-I-1)}\mathbb{E}_{\widetilde{P}_I(X, \widetilde{Y})}\mathbb{E}_{\sigma}\Biggl[\sup_{f\in\mathcal{F}}\frac{1}{n}\sum_{i=1}^n \sigma_i\sum_{y\in \mathcal{Y}}\mathcal{L}(f(x_i), y)
    \Biggr] \nonumber \\
    &\le \frac{K(K-1)}{I(2K-I-1)}\sum_{y\in \mathcal{Y}}\mathbb{E}_{\widetilde{P}_I(X, \widetilde{Y})}\mathbb{E}_{\sigma}\Biggl[\sup_{f\in\mathcal{F}}\frac{1}{n}\sum_{i=1}^n \sigma_i\mathcal{L}(f(x_i), y)
    \Biggr] \nonumber \\
    &= \frac{K(K-1)}{I(2K-I-1)}\sum_{y\in \mathcal{Y}} \mathfrak{R}_n(\mathcal{L}\circ \mathcal{F}(\cdot,y)) \nonumber 
    \\
    &\le \frac{\sqrt{2}\rho K^2(K-1)}{I(2K-I-1)}\sum_{y\in \mathcal{Y}} \mathfrak{R}_n(\mathcal{G}_y), \label{eq:rade_H-w-o}
\end{align}
where we applied the Rademacher vector contraction inequality \citep{10.1007/978-3-319-46379-7_1} to the last inequality. 
\subsection*{Proof of Theorem \ref{thm:which_one_error}}
First, the following holds:
\begin{align}
    &\sup_{f \in \mathcal{F}}\Biggl\{\widetilde{\mathcal{L}}_{I}(f(x), \tilde{y}) - \widetilde{\mathcal{L}}_{I}(f(x'), \tilde{y}')\Biggr\} \nonumber \\
    &\hspace{3em}= \sup_{f \in \mathcal{F}}\Biggl[\Biggl\{\sum_{y\in\tilde{y}}\mathcal{L}(f(x), y) - \frac{(K-I)(K-I-1)}{I(2K-I-1)}\sum_{y\notin \tilde{y}}\mathcal{L}(f(x), y) \Biggr\} \nonumber \\
    &\hspace{6em}- \Biggl\{\sum_{y\in\tilde{y}'}\mathcal{L}(f(x'), y) - \frac{(K-I)(K-I-1)}{I(2K-I-1)}\sum_{y\notin \tilde{y}'}\mathcal{L}(f(x'), y) \Biggr\}\Biggr] \nonumber \\
    &\hspace{3em}= \sup_{f \in \mathcal{F}}\Biggl[\Biggl\{\sum_{y\in\tilde{y}}\mathcal{L}(f(x), y) - \sum_{y\in\tilde{y}'}\mathcal{L}(f(x'), y) \Biggr\} \nonumber \\
    &\hspace{6em}- \frac{(K-I)(K-I-1)}{I(2K-I-1)}\Biggl\{\sum_{y\notin \tilde{y}}\mathcal{L}(f(x), y)  - \sum_{y\notin \tilde{y}'}\mathcal{L}(f(x'), y) \Biggr\}\Biggr] \nonumber \\
    &\hspace{3em}\le \left\{|\tilde{y}|+\frac{(K-|\tilde{y}'|)(K-I)(K-I-1)}{I(2K-I-1)}\right\}C_{L} \nonumber \\
    &\hspace{3em}\le \left\{(K-I) + \frac{(K-1)(K-I)(K-I-1)}{I(2K-I-1)}\right\}C_{L} \nonumber\\
    &\hspace{3em}= \left\{\frac{(K-I)(K^2+(I-2)K-I^2+1)}{I(2K-I-1)}\right\}C_{L}.
    \label{eq:sup_L-L}
\end{align}
Next, let us define function $A((x_1, \tilde{y}_1), (x_2, \tilde{y}_2),\ldots, (x_n, \tilde{y}_n))$ as follows:
\begin{align*}
    &A((x_1, \tilde{y}_1), (x_2, \tilde{y}_2),\ldots, (x_n, \tilde{y}_n))\\
    &= \sup_{f \in \mathcal{F}}\biggl\{\hat{R}(f) - R(f)\biggl\} \\
    &= \sup_{f \in \mathcal{F}}\biggl\{\frac{1}{n}\sum_{i=1}^n \widetilde{\mathcal{L}}_{I}(f(x_i), \tilde{y}_i) - \mathbb{E}_{\widetilde{P}_I(X, \widetilde{Y})}\biggl[\widetilde{\mathcal{L}}_{I}(f(X), \widetilde{Y})\biggl]\biggl\}. 
\end{align*}
Then, the following inequality holds:
\begin{align*}
    &A((x_1, \tilde{y}_1), (x_2, \tilde{y}_2),..., (x_n, \tilde{y}_n)) - A((x_1, \tilde{y}_1), (x_2, \tilde{y}_2),..., (x_{n-1}, \tilde{y}_{n-1}), (x', \tilde{y}')) \nonumber \\
    &= \sup_{f \in \mathcal{F}}\inf_{f' \in \mathcal{F}}
    \biggl\{\frac{1}{n}\sum_{i=1}^n\widetilde{\mathcal{L}}_{I}(f(x_i), \tilde{y}_i) - \mathbb{E}_{P_I(X, \widetilde{Y})}\biggl[\widetilde{\mathcal{L}}_{I}(f(X), \widetilde{Y})\biggl] \nonumber \\
    &\hspace{6em}- \biggl(\frac{1}{n}\sum_{i=1}^{n-1}\widetilde{\mathcal{L}}_{I}(f'(x_i), \tilde{y}_i) + \frac{1}{n}\widetilde{\mathcal{L}}_{I}(f'(x'), \tilde{y}') - \mathbb{E}_{P_I(X, \widetilde{Y})}\biggl[\widetilde{\mathcal{L}}_{I}(f'(X), \widetilde{Y})\biggl]\biggr)\biggl\} \nonumber \\
    &\le \sup_{f \in \mathcal{F}}
    \biggl\{\frac{1}{n}\sum_{i=1}^n\widetilde{\mathcal{L}}_{I}(f(x_i), \tilde{y}_i) - \mathbb{E}_{P_I(X, \widetilde{Y})}\biggl[\widetilde{\mathcal{L}}_{I}(f(X), \widetilde{Y})\biggl] \nonumber \\
    &\hspace{6em}- \biggl(\frac{1}{n}\sum_{i=1}^{n-1}\widetilde{\mathcal{L}}_{I}(f(x_i), \tilde{y}_i) + \frac{1}{n}\widetilde{\mathcal{L}}_{I}(f(x'), \tilde{y}') - \mathbb{E}_{P_I(X, \widetilde{Y})}\biggl[\widetilde{\mathcal{L}}_{I}(f(X), \widetilde{Y})\biggl]\biggr)\biggl\} \nonumber \\
    &= \frac{1}{n}\sup_{f \in \mathcal{F}}\biggl\{\widetilde{\mathcal{L}}_{I}(f(x_n), \tilde{y}_n) - \widetilde{\mathcal{L}}_{I}(f(x'), \tilde{y}')\biggl\} \nonumber \\
    &\le \biggl\{\frac{(K-I)(K^2+(I-2)K-I^2+1)}{nI(2K-I-1)}\biggl\}C_{L}.
\end{align*}
For the last equality, we used Eq.\ \eqref{eq:sup_L-L}.
Similarly, the following holds:
\begin{align*}
    &A((x_1, \tilde{y}_1), (x_2, \tilde{y}_2),..., (x_{n-1}, \tilde{y}_{n-1}), (x', \tilde{y}')) - A((x_1, \tilde{y}_1), (x_2, \tilde{y}_2),..., (x_n, \tilde{y}_n))\\
    &\le \biggl\{\frac{(K-I)(K^2+(I-2)K-I^2+1)}{nI(2K-I-1)}\biggl\}C_{L}.
\end{align*}
Summarizing the above two inequalities, the following holds.
\begin{align*}
    &|A((x_1, \tilde{y}_1), (x_2, \tilde{y}_2),..., (x_n, \tilde{y}_n)) - A((x_1, \tilde{y}_1), (x_2, \tilde{y}_2),..., (x_{n-1}, \tilde{y}_{n-1}), (x', \tilde{y}'))|\\
    &\le \biggl\{\frac{(K-I)(K^2+(I-2)K-I^2+1)}{nI(2K-I-1)}\biggl\}C_{L}.
\end{align*}
Therefore, from McDiarmid's inequality, the following holds for the probability larger than $1-\delta$.
\begin{align*}
    &\sup_{f \in \mathcal{F}}\biggl\{\hat{R}(f)-R(f)\biggr\} - \mathbb{E}\biggl[\sup_{f \in \mathcal{F}}\biggl\{\hat{R}(f)-R(f)\biggr\}\biggr]\nonumber\\
    &\le \frac{(K-I)(K^2+(I-2)K-I^2+1)}{2I(2K-I-1)}C_{L}\sqrt{\frac{\log{(1/\delta})}{2n}}.
\end{align*}
Let $\mathcal{S}:=\left\{(x_{i}, \tilde{y}_{i})\right\}_{i=1}^{n}$ and  $\mathcal{S}':=\left\{(x'_{i}, \tilde{y}'_{i})\right\}_{i=1}^{n}$ be a dataset where each dataset is drawn from the generative model $\widetilde{P}_I(X, \widetilde{Y})$.
Then, because $R(f)=\mathbb{E}[\hat{R}(f)]$, the following holds:
\begin{align}
    &\mathbb{E}\left[\sup_{f \in \mathcal{F}}\left(\hat{R}(f)-R(f)\right)\right]\nonumber\\
    &=\mathbb{E}_{\mathcal{S}}\left[\sup_{f \in \mathcal{F}}\left(
    \frac{1}{n}\sum_{i=1}^{n}\widetilde{\mathcal{L}}_{I}(f(x_i), \tilde{y}_i)
    -\mathbb{E}_{\mathcal{S}'}\left[\frac{1}{n}\sum_{i=1}^{n}\widetilde{\mathcal{L}}_{I}(f(x'_i), \tilde{y}'_i)\right]\right)\right]\nonumber\\
    &\le\mathbb{E}_{\mathcal{S}}\mathbb{E}_{\mathcal{S}'}\left[\sup_{f \in \mathcal{F}}\left(
    \frac{1}{n}\sum_{i=1}^{n}\widetilde{\mathcal{L}}_{I}(f(x_i), \tilde{y}_i)
    -\frac{1}{n}\sum_{i=1}^{n}\widetilde{\mathcal{L}}_{I}(f(x'_i), \tilde{y}'_i)\right)\right]\nonumber\\
    &=\mathbb{E}_{\mathcal{S}}\mathbb{E}_{\mathcal{S}'}\left[\sup_{f \in \mathcal{F}}\left(
    \frac{1}{n}\sum_{i=1}^{n}\widetilde{\mathcal{L}}_{I}(f(x_i), \tilde{y}_i)
    -\frac{1}{n}\sum_{i=1}^{n}\widetilde{\mathcal{L}}_{I}(f(x'_i), \tilde{y}'_i)\right)\right]\nonumber\\
    &=\mathbb{E}_{\mathcal{S}}\mathbb{E}_{\mathcal{S}'}\left[\sup_{f \in \mathcal{F}}\frac{1}{n}\sum_{i=1}^{n}\left(
    \widetilde{\mathcal{L}}_{I}(f(x_i), \tilde{y}_i)
    -\widetilde{\mathcal{L}}_{I}(f(x'_i), \tilde{y}'_i)\right)\right]\nonumber\\
    &=\mathbb{E}_{\mathcal{S}}\mathbb{E}_{\mathcal{S}'}\mathbb{E}_{\sigma}\left[\sup_{f \in \mathcal{F}}\frac{1}{n}\sum_{i=1}^{n}\sigma_{i}\left(
    \widetilde{\mathcal{L}}_{I}(f(x_i), \tilde{y}_i)
    -\widetilde{\mathcal{L}}_{I}(f(x'_i), \tilde{y}'_i)\right)\right]\label{eq:EsupR-R_wo_1}\\
    &\le\mathbb{E}_{\mathcal{S}}\mathbb{E}_{\sigma}\left[\sup_{f \in \mathcal{F}}\frac{1}{n}\sum_{i=1}^{n}\sigma_{i}\widetilde{\mathcal{L}}_{I}(f(x_i), \tilde{y}_i)\right]
    +\mathbb{E}_{\mathcal{S}'}\mathbb{E}_{\sigma}\left[\sup_{f \in \mathcal{F}}\frac{1}{n}\sum_{i=1}^{n}(-\sigma_{i})\widetilde{\mathcal{L}}_{I}(f(x'_i), \tilde{y}'_i)\right]\nonumber\\
    &=2\mathfrak{R}_n(\widetilde{\mathcal{H}})\nonumber\\
    &\le \frac{2\sqrt{2}\rho K^2(K-1)}{I(2K-I-1)}\sum_{y\in \mathcal{Y}} \mathfrak{R}_n(\mathcal{G}_y).\label{eq:EsupR-R_wo_2}
\end{align}
In Eq.\ \eqref{eq:EsupR-R_wo_1}, we used the fact that $\widetilde{\mathcal{L}}_{I}(f(x_i), \tilde{y}_i)-\widetilde{\mathcal{L}}_{I}(f(x'_i), \tilde{y}'_i)$ and $\sigma_i(\widetilde{\mathcal{L}}_{I}(f(x_i), \tilde{y}_i)-\widetilde{\mathcal{L}}_{I}(f(x'_i), \tilde{y}'_i))$ follow the same distribution.
For Eq.\ \eqref{eq:EsupR-R_wo_2}, we used Eq.\ \eqref{eq:rade_H-w-o}.
Thus, the following holds for the classification error $\mathcal{E}_{I}$.
\begin{align*}
    \mathcal{E}_{I} &:= R(\hat{f})-R(f^*) \\
    &= (\hat{R}(\hat{f})-\hat{R}(f^*)) + (R(\hat{f})-\hat{R}(\hat{f})) + (\hat{R}(f^*)-R(f^*)) \\
    &\le 2\sup_{f \in \mathcal{F}}\mid\hat{R}(f)-R(f)\mid \\
    &\le \frac{4\sqrt{2}\rho K^2(K-1)}{I(2K-I-1)}\sum_{y\in \mathcal{Y}} \mathfrak{R}_n(\mathcal{G}_y) + \frac{(K-I)(K^2+(I-2)K-I^2+1)}{I(2K-I-1)}C_{L}\sqrt{\frac{2\log{(2/\delta})}{n}}.
\end{align*}
\section*{Appendix B: Is-in-type Q$\&$A Labeling}\label{sec:calculate2}
\subsection*{Proof of Theorem \ref{thm:is-in-gen}}
\begin{lemma}\label{lem:is-in-y-given-x-and-z}
When the class assumed by the annotator for the instance $x$ is $z$, the
generative model of the is-in-type Q$\&$A label for that instance is given as follows:
\begin{align}
\Pr\{Y=q|X=x, Z=z\} &=
\begin{cases}
\displaystyle\frac{1}{{}_K C_I}, & \text{if $z \in q$} \\
0, & \text{if $z \not\in q$} \\
\end{cases},
\quad \forall q \in \mathfrak{B}_{I}(\mathcal{Y}),
\label{eq:is-in-gen-1}
\\
\Pr\{Y=\mathcal{Y}\backslash q|X=x, Z=z\} &=
\begin{cases}
0, & \text{if $z \in q$} \\
\displaystyle\frac{1}{{}_K C_I}, & \text{if $z \not\in q$} \\
\end{cases},
\quad \forall q \in \mathfrak{B}_{I}(\mathcal{Y}).
\label{eq:is-in-gen-2}
\end{align}
\end{lemma}
\begin{proof}
First, we derived Eq.\ \eqref{eq:is-in-gen-1}.
This is trivial for $z \not\in q$.
For $z \in q$, the following holds:
\begin{align*}
    \Pr\bigl\{Y=q|X=x, Z=z\bigl\} = \Pr\bigl\{z \in q|X=x, Z=z, Q=q\bigl\}\Pr\bigl\{Q=q\bigl\} = \frac{1}{\comb{K}{I}},
\end{align*}
where we used that $\Pr\bigl\{z \in q|X=x, Z=z, Q=q\bigl\}=1$ when $z \in q$.

Next, we derive Eq.\ \eqref{eq:is-in-gen-2}.
This is trivial for $z \in q$.
For $z \not\in q$, the following holds:
\begin{align*}
    \Pr\bigl\{Y=\mathcal{Y}\backslash q|X=x, Z=z\bigl\} &= \Pr\bigl\{z \not\in q|X=x, Z=z, Q=q\bigl\}\Pr\bigl\{Q=q\bigl\} = \frac{1}{\comb{K}{I}},
\end{align*}
where we used that $\Pr\bigl\{z \not\in q|X=x, Z=z, Q=q\bigl\}=1$ when $z \not\in q$.
\end{proof}

Using \ref{lem:is-in-y-given-x-and-z}, we can prove Theorem \ref{thm:is-in-gen}.
First, for the instance, $x$, the probability that label $q(\in \mathfrak{B}_I(\mathcal{Y}))$ can be derived as follows:
\begin{align*}
    \Pr\bigl\{Y=q|X=x\bigl\} &= \sum_{z \in \mathcal{Y}}\Pr\bigl\{Y=q|X=x, Z=z\}\Pr\bigl\{Z=z|X=x\bigl\} \\ 
    &= \frac{1}{\comb{K}{I}}\sum_{z \in q}\Pr\{Z=z|X=x\}. \nonumber
\end{align*}

In addition, the probability that label $\mathcal{Y}\backslash q(\in \mathfrak{B}_{K-I}(\mathcal{Y}))$ is assigned to the instance $x$ can be derived as follows:
\begin{align*}
    \Pr\bigl\{Y=\mathcal{Y}\backslash q|X=x\bigl\} &= \sum_{z \in \mathcal{Y}}\Pr\bigl\{Y=\mathcal{Y}\backslash q|X=x, Z=z\}\Pr\bigl\{Z=z|X=x\bigl\} \\ 
    &= \frac{1}{\comb{K}{I}}\sum_{z \in \mathcal{Y}\backslash q}\Pr\{Z=z|X=x\}.
\end{align*}

\subsection*{Proof of Corollary \ref{cor:is-in-recv}}
First, the following holds:
\begin{align}
    &\Pr\{\widehat{Y}=\alpha|X=x\}\nonumber\\
    &= \Pr\{\widehat{Y} =\alpha|X=x, \alpha \in Y \in \mathfrak{B}_I(\mathcal{Y})\}\Pr\{\alpha \in Y \in \mathfrak{B}_I(\mathcal{Y})|X=x\} \nonumber \\
    &\hspace{2em}+ \Pr\{\widehat{Y} =\alpha|X=x, \alpha \notin Y \in \mathfrak{B}_I(\mathcal{Y})\}\Pr\{\alpha \notin Y \in \mathfrak{B}_I(\mathcal{Y})|X=x\} \nonumber \\
    &\hspace{2em}+ \Pr\{\widehat{Y} =\alpha|X=x, \alpha \in Y \in \mathfrak{B}_{K-I}(\mathcal{Y})\}\Pr\{\alpha \in Y \in \mathfrak{B}_{K-I}(\mathcal{Y})|X=x\} \nonumber \\
    &\hspace{2em}+ \Pr\{\widehat{Y} =\alpha|X=x, \alpha \notin Y \in \mathfrak{B}_{K-I}(\mathcal{Y})\}\Pr\{\alpha \notin Y \in \mathfrak{B}_{K-I}(\mathcal{Y})|X=x\} \nonumber \\
    &= \Pr\{\widehat{Y} =\alpha|X=x, \alpha \in Y \in \mathfrak{B}_I(\mathcal{Y})\}\Pr\{\alpha \in Y \in \mathfrak{B}_I(\mathcal{Y})|X=x\} \nonumber \\
    &\hspace{2em}+ \Pr\{\widehat{Y} =\alpha|X=x, \alpha \in Y \in \mathfrak{B}_{K-I}(\mathcal{Y})\}\Pr\{\alpha \in Y \in \mathfrak{B}_{K-I}(\mathcal{Y})|X=x\} \nonumber \\
    &= \frac{1}{I}\Pr\{\alpha \in Y \in \mathfrak{B}_I(\mathcal{Y})|X=x\} + \frac{1}{K-I}\Pr\{\alpha \in Y \in \mathfrak{B}_{K-I}(\mathcal{Y})|X=x\}. \label{eq:is-in-recv-1}
\end{align}
In the third equality, we used $\Pr\{\widehat{Y} =\alpha|X=x, \alpha \notin Y \in \mathfrak{B}_I(\mathcal{Y})\} = 0$ $\Pr\{\widehat{Y} =\alpha|X=x, \alpha \notin Y \in \mathfrak{B}_{K-I}(\mathcal{Y})\}=0$.
In the last equality, we used $\Pr\{\widehat{Y} =\alpha|X=x, \alpha \in Y \in \mathfrak{B}_I(\mathcal{Y})\}=1/I$, $\Pr\{\widehat{Y} =\alpha|X=x, \alpha \in Y \in \mathfrak{B}_{K-I}(\mathcal{Y})\}=1/(K-I)$.

For the first term on the right-hand side of Eq.\ \eqref{eq:is-in-recv-1}, the following
holds.
\begin{align}
    \Pr\bigl\{\alpha \in Y \in \mathfrak{B}_I(\mathcal{Y})|X=x\bigl\}
    &\equiv \sum_{\tilde{y} \in \mathfrak{B}_I(\mathcal{Y})}\mathbbm{1}(\alpha \in \tilde{y})\widetilde{P}_I(\tilde{y}|x) \nonumber \\
    &= \frac{1}{{}_K C_I}\sum_{\tilde{y} \in \mathfrak{B}_I(\mathcal{Y})}\mathbbm{1}(\alpha \in \tilde{y}^{(I)})\sum_{y \in \tilde{y}}P(y|x) \nonumber \\
    &= \frac{1}{{}_K C_I}\{{}_{K-2} C_{I-1} P(\alpha|x) + {}_{K-2} C_{I-2}\} \nonumber \\
    &= \frac{I(K-I)}{K(K-1)}P(\alpha|x) + \frac{I(I-1)}{K(K-1)}.
    \label{eq:is-in-recv-2}
\end{align}
In the second equality, we used Theorem \ref{thm:which-one-gen}.
For the third equality, we used Eq.\ \eqref{eq:which-one-recv-4}.
The same relationship as in Eq.\ \eqref{eq:which-one-recv-3} holds for the second term on the right-hand side of Eq.\ \eqref{eq:is-in-recv-1}.
Substituting Eqs.\ \eqref{eq:is-in-recv-2} and \eqref{eq:which-one-recv-3} into Eq.\ \eqref{eq:is-in-recv-1}, Corollary \ref{cor:is-in-recv} holds as follows.
\begin{align*}
    &\Pr\{\widehat{Y}=\alpha|X=x\}\nonumber\\
    &= \frac{K-I}{K(K-1)}P(\alpha|x) + \frac{(I-1)}{K(K-1)} + \frac{I}{K(K-1)}P(\alpha|x) + \frac{(K-I-1)}{K(K-1)}\\
    &=\frac{1}{K-1}\Pr(\alpha|x) + \frac{K-2}{K-1} \cdot \frac{1}{K}.
\end{align*}

\subsection*{Proof of Theorem \ref{thm:is-in-loss}}
\begin{lemma}\label{lem:is-in-pdf}
The probability that an ordinary label $y$ is assigned to an instance $x$, $P(y|x)$, can be expressed using the probability of an is-in-type $Q\&A$ Label $\tilde{y}$ being assigned for the given instance $x$ as follows:
\begin{align}
    P(y|x)
    &= \frac{K(K-1)}{2I(K-I)}
    \sum_{y \in \tilde{y}\in \mathfrak{B}_{I}(\mathcal{Y}) \cup \mathfrak{B}_{K-I}(\mathcal{Y})}\widetilde{P}_I(\tilde{y}|x) - \frac{2I^2+K^2-K(2I+1)}{2I(K-I)}\nonumber \\
    &= \frac{K(K-1)}{2I(K-I)}
    \sum_{y \in \tilde{y}\in \mathfrak{B}_{I}(\mathcal{Y}) \cup \mathfrak{B}_{K-I}(\mathcal{Y})}\widetilde{P}_I(\tilde{y}|x) + 1 - \frac{K(K-1)}{2I(K-I)}.
    \label{eq:is-in-pdf}
\end{align}
\end{lemma}
\begin{proof}
First, the following holds:
\begin{align}
    &\Pr\{y \in \widetilde{Y} \in \mathfrak{B}_{I}(\mathcal{Y}) \cup \mathfrak{B}_{K-I}(\mathcal{Y})|X=x\}\nonumber\\
    &= \sum_{y \in \tilde{y} \in \mathfrak{B}_{I}(\mathcal{Y}) \cup \mathfrak{B}_{K-I}(\mathcal{Y})} \widetilde{P}_I(\tilde{y}|x)
    = \sum_{y \in \tilde{y} \in \mathfrak{B}_{I}(\mathcal{Y})}\tilde{P}_I(\tilde{y}|x) + \sum_{y \in \tilde{y} \in \mathfrak{B}_{I}(\mathcal{Y})} \tilde{P}_{K-I}(\tilde{y}|x).
    \label{eq:is-in-pdf-proof}
\end{align}

According to Theorem \ref{thm:is-in-gen}, for the first term on the right-hand side of the Eq.\ \eqref{eq:is-in-pdf-proof}, the following holds:
\begin{align*}
    \sum_{y\in\tilde{y}\in \mathfrak{B}_I(\mathcal{Y})}\widetilde{P}_I(\widetilde{Y}|x)
    &= \frac{1}{\comb{K}{I}}\sum_{y \in y' \in \mathfrak{B}_I(\mathcal{Y})}\sum_{y' \in \tilde{y}}P(y'|x) \nonumber \\
    &= \frac{1}{\comb{K}{I}}\sum_{\tilde{y}\in \mathfrak{B}_I(\mathcal{Y})}\mathbbm{1}(y\in\tilde{y})\sum_{y'\in\tilde{y}}P(y'|x)\\
    &= \frac{1}{\comb{K}{I}}\{\comb{K-2}{I-1}P(y|x) + \comb{K-2}{I-2}\} \nonumber \\
    &= \frac{I(K-I)}{K(K-1)}P(y|x) + \frac{I(I-1)}{K(K-1)}.
\end{align*}
For the third equality, we used Eq.\ \eqref{eq:which-one-recv-4}.

Similarly, according to Theorem \ref{thm:is-in-gen}, for the second term on the right-hand side of Eq.\ \eqref{eq:is-in-pdf-proof}, the following holds:
\begin{align*}
    \sum_{y \in \tilde{y} \in \mathfrak{B}_{K-I}(\mathcal{Y})}\widetilde{P}_I(\tilde{y}|x)
    &= \frac{1}{\comb{K}{I}}\sum_{\tilde{y}\in \mathfrak{B}_{K-I}(\mathcal{Y})^y}\sum_{y' \in \tilde{y}}P(y'|x) \nonumber \\
    &= \frac{1}{\comb{K}{I}}\sum_{\tilde{y}\in \mathfrak{B}_{K-I}(\mathcal{Y})}\mathbbm{1}(y\in\tilde{y})\sum_{y'\in\tilde{y}}P(y'|x)\\
    &= \frac{1}{\comb{K}{I}}\{\comb{K-2}{K-I-1}P(y|x) + \comb{K-2}{K-I-2}\} \nonumber \\
    &= \frac{I(K-I)}{K(K-1)}P(y|x) + \frac{(K-I)(K-I-1)}{K(K-1)}.
\end{align*}
For the third equality, we used Eq.\ \eqref{eq:which-one-recv-4}.

From the above, the following holds:
\begin{align*}
    &\sum_{y \in \tilde{y} \in \mathfrak{B}_{I}(\mathcal{Y}) \cup \mathfrak{B}_{K-I}(\mathcal{Y})} \widetilde{P}_I(\tilde{y}|x)\\
    &= \frac{I(K-I)}{K(K-1)}P(y|x) + \frac{I(I-1)}{K(K-1)} + \frac{I(K-I)}{K(K-1)}P(y|x) + \frac{(K-I)(K-I-1)}{K(K-1)}\\
    &= \frac{2I(K-I)}{K(K-1)}P(y|x) + \frac{2I^2+K^2-K(2I+1)}{K(K-1)}.
\end{align*}
By rearranging this result, Eq.\ \eqref{eq:is-in-pdf} is derived.
\end{proof}

For the classification risk defined in Eq.\ \eqref{eq:pred_err}, the following holds:
\begin{align}
    R(f) = \mathbb{E}_{P(X)}\Bigl[\mathbb{E}_{P(Y|X)}\Bigl[\mathcal{L}(f(x), y)\Bigl]\Bigl]. \label{eq:is-in_expected}
\end{align}
Using Lemma \ref{lem:is-in-pdf}, the following holds:
\begin{align*}
    &\mathbb{E}_{P(Y|x)}\Bigl[\mathcal{L}(f(x), y)\Bigl]\\
    &= \sum_{y \in \mathcal{Y}}\mathcal{L}(f(x), y)P(y|x)\\
    &= \sum_{y \in \mathcal{Y}}\mathcal{L}(f(x), y)\Biggl\{\frac{K(K-1)}{2I(K-I)}
    \sum_{y \in \tilde{y}\in \mathfrak{B}_{I}(\mathcal{Y}) \cup \mathfrak{B}_{K-I}(\mathcal{Y})}\widetilde{P}_I(\tilde{y}|x) - \frac{2I^2+K^2-K(2I+1)}{2I(K-I)}\Biggr\} \nonumber \\
    &= \frac{K(K-1)}{2I(K-I)}\sum_{y \in \mathcal{Y}}\sum_{y \in \tilde{y}\in \mathfrak{B}_{I}(\mathcal{Y}) \cup \mathfrak{B}_{K-I}(\mathcal{Y})}\mathcal{L}(f(x), y)\widetilde{P}_I(\tilde{y}|x) - \frac{2I^2+K^2-K(2I+1)}{2I(K-I)}\sum_{y \in \mathcal{Y}}\mathcal{L}(f(x), y)\\
    &= \frac{K(K-1)}{2I(K-I)}\sum_{\tilde{y}\in \mathfrak{B}_{I}(\mathcal{Y}) \cup \mathfrak{B}_{K-I}(\mathcal{Y})}\sum_{y \in \tilde{y}}\mathcal{L}(f(x), y)\widetilde{P}_I(\tilde{y}|x) - \frac{2I^2+K^2-K(2I+1)}{2I(K-I)}\sum_{y \in \mathcal{Y}}\mathcal{L}(f(x), y)\\
    &= \frac{K(K-1)}{2I(K-I)}\sum_{\tilde{y}\in \mathfrak{B}_{I}(\mathcal{Y}) \cup \mathfrak{B}_{K-I}(\mathcal{Y})}\widetilde{P}_I(\tilde{y}|x)\sum_{y \in \tilde{y}}\mathcal{L}(f(x), y') - \frac{2I^2+K^2-K(2I+1)}{2I(K-I)}\sum_{y \in \mathcal{Y}}\mathcal{L}(f(x), y)\\
    &= \mathbb{E}_{\widetilde{P}_I(\widetilde{Y}|x)}\Biggl[\frac{K(K-1)}{2I(K-I)}\sum_{y \in \widetilde{Y}}\mathcal{L}(f(x), \widetilde{Y}) - \frac{2I^2+K^2-K(2I+1)}{2I(K-I)}\sum_{y \in \mathcal{Y}}\mathcal{L}(f(x), y)\Biggr]\\
    &= \mathbb{E}_{\widetilde{P}_I(Y|x)}\Biggl[\frac{K(K-1)}{2I(K-I)}\sum_{y \in \widetilde{Y}}\mathcal{L}(f(x), \widetilde{Y})\\
    &\hspace{5em}- \frac{2I^2+K^2-K(2I+1)}{2I(K-I)}\sum_{y \in \widetilde{Y}}\mathcal{L}(f(x), y) - \frac{2I^2+K^2-K(2I+1)}{2I(K-I)}\sum_{y \notin \widetilde{Y}}\mathcal{L}(f(x), y)\Biggr]\\
    &= \mathbb{E}_{\widetilde{P}_I(\widetilde{Y}|x)}\left[\sum_{y \in \widetilde{Y}}\mathcal{L}(f(x), y) - \frac{2I^2+K^2-K(2I+1)}{2I(K-I)}\sum_{y \notin \widetilde{Y}}\mathcal{L}(f(x), y)\right].
\end{align*}
Therefore,
\begin{align*}
    \widetilde{\mathcal{L}}_{I}(f(X), \widetilde{Y}) &=
       \sum_{y \in \widetilde{Y}}\mathcal{L}(f(x), y) - \frac{2I^2+K^2-K(2I+1)}{2I(K-I)}\sum_{y \notin \widetilde{Y}}\mathcal{L}(f(x), y).
\end{align*}
Subsequently, by substituting this into Eq.\ \eqref{eq:is-in_expected}, the following holds:
\begin{align*}
    R(f) = \mathbb{E}_{\widetilde{P}_I(X, \widetilde{Y})}\Bigl[\widetilde{\mathcal{L}}_{I}(f(X), \widetilde{Y})\Bigl].
\end{align*}

\subsection*{Proof of Lemma \ref{lem:is_in_rade_loss}}
If $\alpha_{i} = 2\mathbbm{1}(y \in \tilde{y}_{i})-1$, for the Rademacher complexity $\mathfrak{R}_n(\widetilde{\mathcal{H}})$ of function set $\widetilde{\mathcal{H}}$ the following holds:
\begin{align}
    &\mathfrak{R}_n(\widetilde{\mathcal{H}})\nonumber\\
    &= \mathbb{E}_{\widetilde{P}_I(X, \widetilde{Y})}\mathbb{E}_{\sigma}\Biggl[\sup_{h\in\widetilde{\mathcal{H}}}\frac{1}{n}\sum_{i=1}^n\sigma_i h(x_i, \tilde{y}_i)\Biggr] \nonumber \\
    &= \mathbb{E}_{\widetilde{P}_I(X, \widetilde{Y})}\mathbb{E}_{\sigma}\Biggl[\sup_{f\in\mathcal{F}}\frac{1}{n}\sum_{i=1}^n \sigma_i \Biggl\{ \sum_{y\in \tilde{y}_i}\mathcal{L}(f(x_i), y) - \frac{2I^2+K^2-K(2I+1)}{2I(K-I)}\sum_{y \notin \tilde{y}_i}\mathcal{L}(f(x_i), y)\Biggr\}\Biggr] \nonumber \\
    &\le \mathbb{E}_{\widetilde{P}_I(X, \widetilde{Y})}\mathbb{E}_{\sigma}\Biggl[\sup_{f\in\mathcal{F}}\frac{1}{n}\sum_{i=1}^n \sigma_i \sum_{y\in \tilde{y}_i}\mathcal{L}(f(x_i), y)\Biggr] \nonumber \\ 
    &\hspace{1em}+\mathbb{E}_{\widetilde{P}_I(X, \widetilde{Y})}\mathbb{E}_{\sigma}\Biggl[\sup_{f\in \mathcal{F}} \frac{1}{n}\sum_{i=1}^n\sigma_i\frac{2I^2+K^2-K(2I+1)}{2I(K-I)}\sum_{y \notin \tilde{y}_i}\mathcal{L}(f(x_i), y)\Biggr] \nonumber \\
    &= \mathbb{E}_{\widetilde{P}_I(X, \widetilde{Y})}\mathbb{E}_{\sigma}\Biggl[\sup_{f\in\mathcal{F}}\frac{1}{2n}\sum_{i=1}^n \sigma_i \sum_{y\in \mathcal{Y}}\mathcal{L}(f(x_i), y)(\alpha_i + 1)\Biggr] \nonumber \\ 
    &\hspace{1em}+\mathbb{E}_{\widetilde{P}_I(X, \widetilde{Y})}\mathbb{E}_{\sigma}\Biggl[\sup_{f\in \mathcal{F}} \frac{1}{2n}\sum_{i=1}^n\sigma_i\frac{2I^2+K^2-K(2I+1)}{2I(K-I)}\sum_{y \in \mathcal{Y}}\mathcal{L}(f(x_i), y)(1-\alpha_i)\Biggr] \nonumber \\
    &= \mathbb{E}_{\widetilde{P}_I(X, \widetilde{Y})}\mathbb{E}_{\sigma}\Biggl[\sup_{f\in\mathcal{F}}\frac{1}{2n}\sum_{i=1}^n \Biggl(\sigma_i \sum_{y\in \mathcal{Y}}\alpha_i\mathcal{L}(f(x_i), y) + \sigma_i \sum_{y\in \mathcal{Y}}\mathcal{L}(f(x_i), y)
    \Biggr)\Biggr] \nonumber \\ 
    &\hspace{1em}+\frac{2I^2+K^2-K(2I+1)}{2I(K-I)}\mathbb{E}_{\widetilde{P}_I(X, \widetilde{Y})}\mathbb{E}_{\sigma}\Biggl[\sup_{f\in \mathcal{F}} \frac{1}{2n}\sum_{i=1}^n\Biggl(\sigma_i\sum_{y \in \mathcal{Y}}\mathcal{L}(f(x_i), y) - \sigma_i\sum_{y \in \mathcal{Y}}\alpha_i\mathcal{L}(f(x_i), y)\Biggr)\Biggr] \nonumber \\
    &\le \mathbb{E}_{\widetilde{P}_I(X, \widetilde{Y})}\mathbb{E}_{\sigma}\Biggl[\sup_{f\in\mathcal{F}}\frac{1}{n}\sum_{i=1}^n \sigma_i \sum_{y\in \mathcal{Y}}\mathcal{L}(f(x_i), y)
    \Biggr] \nonumber \\
    &\hspace{1em}+\frac{2I^2+K^2-K(2I+1)}{2I(K-I)}\mathbb{E}_{\widetilde{P}_I(X, \widetilde{Y})}\mathbb{E}_{\sigma}\Biggl[\sup_{f\in \mathcal{F}} \frac{1}{n}\sum_{i=1}^n\sigma_i\sum_{y \in \mathcal{Y}}\mathcal{L}(f(x_i), y)\Biggr] \nonumber \\
    &= \frac{K(K-1)}{2I(K-I)}\mathbb{E}_{\widetilde{P}_I(X, \widetilde{Y})}\mathbb{E}_{\sigma}\Biggl[\sup_{f\in\mathcal{F}}\frac{1}{n}\sum_{i=1}^n \sigma_i\sum_{y\in \mathcal{Y}}\mathcal{L}(f(x_i), y)
    \Biggr] \nonumber \\
    &\le \frac{K(K-1)}{2I(K-I)}\sum_{y\in \mathcal{Y}}\mathbb{E}_{\widetilde{P}_I(X, \widetilde{Y})}\mathbb{E}_{\sigma}\biggl[\sup_{f\in\mathcal{F}}\frac{1}{n}\sum_{i=1}^n \sigma_i\mathcal{L}(f(x_i), y)
    \biggl] \nonumber \\
    &= \frac{K(K-1)}{2I(K-I)}\sum_{y\in \mathcal{Y}} \mathfrak{R}_n(\mathcal{L}\circ \mathcal{F}(\cdot,y)) \nonumber \\
    &\le \frac{\sqrt{2}\rho K^2(K-1)}{2I(K-I)}\sum_{y\in \mathcal{Y}} \mathfrak{R}_n(\mathcal{G}_y), \label{eq:rade_H-i-i}
\end{align}
where we applied the Rademacher vector contraction inequality \citep{10.1007/978-3-319-46379-7_1} to the last inequality. 

\subsection*{Proof of Theorem \ref{thm:is_in_error}}
First, the following holds:
\begin{align}
    &\sup_{f \in \mathcal{F}}\Biggl\{\widetilde{\mathcal{L}}^I(f(x), \tilde{y}) - \widetilde{\mathcal{L}}^I(f(x'), \tilde{y}')\Biggr\} \nonumber \\
    &= \sup_{f \in \mathcal{F}}\Biggl[\Biggl\{\sum_{y\in\tilde{y}}\mathcal{L}(f(x), y) - \frac{2I^2+K^2-K(2I+1)}{2I(K-I)}\sum_{y\notin \tilde{y}}\mathcal{L}(f(x), y) \Biggl\} \nonumber \\
    &\hspace{5em}- \Biggl\{\sum_{y\in\tilde{y}'}\mathcal{L}(f(x'), y) - \frac{2I^2+K^2-K(2I+1)}{2I(K-I)}\sum_{y\notin \tilde{y}'}\mathcal{L}(f(x'), y) \Biggl\}\Biggl] \nonumber \\
    &= \sup_{f \in \mathcal{F}}\Biggl[\Biggl\{\sum_{y\in\tilde{y}}\mathcal{L}(f(x), y) - \sum_{y\in\tilde{y}'}\mathcal{L}(f(x'), y) \Biggl\} \nonumber \\
    &\hspace{5em}- \frac{2I^2+K^2-K(2I+1)}{2I(K-I)}\Biggl\{\sum_{y\notin \tilde{y}}\mathcal{L}(f(x), y)  - \sum_{y\notin \tilde{y}'}\mathcal{L}(f(x'), y) \Biggl\}\Biggl] \nonumber \\
    &\le \left\{|\tilde{y}|+\frac{(K-|\tilde{y}'|)(2I^2+K^2-K(2I+1))}{2I(K-I)}\right\}C_{L} \nonumber \\
    &\le \begin{cases}
        \displaystyle\left(I +\frac{2I^2+K^2-K(2I+1)}{2(K-I)}\right)C_{L} & \text{if $I \ge \displaystyle\frac{K}{2}$} \\
        \displaystyle\left((K-I)+\frac{2I^2+K^2-K(2I+1)}{2I}\right)C_{L} & \text{if $I < \displaystyle\frac{K}{2}$}
       \end{cases} \nonumber \\
    &= \begin{cases}
        \displaystyle\frac{K(K-1)}{2(K-I)}C_{L} & \text{if $I \ge \displaystyle\frac{K}{2}$} \\
        \displaystyle\frac{K(K-1)}{2I}C_{L} & \text{if $I < \displaystyle\frac{K}{2}$}
       \end{cases}. \label{eq:sup_L-L_i_i}
\end{align}
Next, let us define the function $A((x_1, \tilde{y}_1), (x_2, \tilde{y}_2),..., (x_n, \tilde{y}_n))$ as follows:
\begin{align*}
    &A((x_1, \tilde{y}_1), (x_2, \tilde{y}_2),..., (x_n, \tilde{y}_n))\\ 
    &= \sup_{f \in \mathcal{F}}\{\hat{R}(f) - R(f)\} \\
    &= \sup_{f \in \mathcal{F}}\biggl\{\frac{1}{n}\sum_{i=1}^n \widetilde{\mathcal{L}}_{I}(f(x_i), \tilde{y}_i) - \mathbb{E}_{\widetilde{P}_I(X, \widetilde{Y})}\biggl[\widetilde{\mathcal{L}}_{I}(f(x), \tilde{y})\biggl]\biggl\}.
\end{align*}
Then, the following inequality holds:
\begin{align*}
    &A((x_1, \tilde{y}_1), (x_2, \tilde{y}_2),..., (x_n, \tilde{y}_n)) - A((x_1, \tilde{y}_1), (x_2, \tilde{y}_2),..., (x_{n-1}, \tilde{y}_{n-1}), (x', \tilde{y}')) \nonumber \\
    &= \sup_{f \in \mathcal{F}}\inf_{f' \in \mathcal{F}}
    \biggl\{\frac{1}{n}\sum_{i=1}^n\widetilde{\mathcal{L}}_{I}(f(x_i), \tilde{y}_i) - \mathbb{E}_{P_I(X, \widetilde{Y})}\biggl[\widetilde{\mathcal{L}}_{I}(f(X), \widetilde{Y})\biggl] \nonumber \\
    &\hspace{3em}- \biggl(\frac{1}{n}\sum_{i=1}^{n-1}\widetilde{\mathcal{L}}_{I}(f'(x_i), \tilde{y}_i) + \frac{1}{n}\widetilde{\mathcal{L}}_{I}(f'(x'), \tilde{y}') - \mathbb{E}_{P_I(X, \widetilde{Y})}\biggl[\widetilde{\mathcal{L}}_{I}(f'(X), \widetilde{Y})\biggl]\biggr)\biggl\} \nonumber \\
    &\le \sup_{f \in \mathcal{F}}
    \biggl\{\frac{1}{n}\sum_{i=1}^n\widetilde{\mathcal{L}}_{I}(f(x_i), \tilde{y}_i) - \mathbb{E}_{P_I(X, \widetilde{Y})}\biggl[\widetilde{\mathcal{L}}_{I}(f(X), \widetilde{Y})\biggl] \nonumber \\
    &\hspace{3em}- \biggl(\frac{1}{n}\sum_{i=1}^{n-1}\widetilde{\mathcal{L}}_{I}(f(x_i), \tilde{y}_i) + \frac{1}{n}\widetilde{\mathcal{L}}_{I}(f(x'), \tilde{y}') - \mathbb{E}_{P_I(X, \widetilde{Y})}\biggl[\widetilde{\mathcal{L}}_{I}(f(X), \widetilde{Y})\biggl]\biggr)\biggl\} \nonumber \\
    &= \frac{1}{n}\sup_{f \in \mathcal{F}}\biggl\{\widetilde{\mathcal{L}}_{I}(f(x_n), \tilde{y}_n) - \widetilde{\mathcal{L}}_{I}(f(x'), \tilde{y}')\biggl\} \nonumber \\
    &= \begin{cases}
        \displaystyle\frac{K(K-1)}{2n(K-I)}C_{L} & \text{if $I \ge \displaystyle\frac{K}{2}$} \\
        \displaystyle\frac{K(K-1)}{2nI}C_{L} & \text{if $I < \displaystyle\frac{K}{2}$}
       \end{cases}.
\end{align*}
For the last equality, we used equation\eqref{eq:sup_L-L_i_i}.
Similarly, the following holds:
\begin{align*}
    &A((x_1, \tilde{y}_1), (x_2, \tilde{y}_2),..., (x_{n-1}, \tilde{y}_{n-1}), (x', \tilde{y}')) - A((x_1, \tilde{y}_1), (x_2, \tilde{y}_2),..., (x_n, \tilde{y}_n))\\
    &\le \begin{cases}
        \displaystyle\frac{K(K-1)}{2n(K-I)}C_{L} & \text{if $I \ge \displaystyle\frac{K}{2}$} \\
        \displaystyle\frac{K(K-1)}{2nI}C_{L} & \text{if $I < \displaystyle\frac{K}{2}$}
       \end{cases}.
\end{align*}
Summarizing these two inequalities, the following holds:
\begin{align*}
    &|A((x_1, \tilde{y}_1), (x_2, \tilde{y}_2),..., (x_n, \tilde{y}_n)) - A((x_1, \tilde{y}_1), (x_2, \tilde{y}_2),..., (x_{n-1}, \tilde{y}_{n-1}), (x', \tilde{y}'))|\\
    &\le \begin{cases}
        \displaystyle\frac{K(K-1)}{2n(K-I)}C_{L} & \text{if $I \ge \displaystyle\frac{K}{2}$} \\
        \displaystyle\frac{K(K-1)}{2nI}C_{L} & \text{if $I < \displaystyle\frac{K}{2}$}
       \end{cases}.
\end{align*}
Therefore, from McDiarmid's inequality, the following holds for the probability larger than $1-\delta$.
\begin{align}
    &\sup_{f \in \mathcal{F}}\mid\hat{R}(f)-R(f)\mid - \mathbb{E}\biggl[\sup_{f \in \mathcal{F}}(\hat{R}(f)-R(f))\biggl]\nonumber\\
    &\hspace{2em}\le 
        \begin{cases}
        \displaystyle\frac{K(K-1)}{4(K-I)}C_{L}\sqrt{\frac{\log{(1/\delta})}{2n}} & \text{if $I \ge \displaystyle\frac{K}{2}$} \nonumber \\
        \displaystyle\frac{K(K-1)}{4I}C_{L}\sqrt{\frac{\log{(1/\delta})}{2n}} & \text{if $I < \displaystyle\frac{K}{2}$}
       \end{cases}.
\end{align}
Let $\mathcal{S}:=\left\{(x_{i}, \tilde{y}_{i})\right\}_{i=1}^{n}$ and $\mathcal{S}':=\left\{(x'_{i}, \tilde{y}'_{i})\right\}_{i=1}^{n}$ be a dataset where each dataset is drawn from the generative model $\widetilde{P}_I(X, \widetilde{Y})$.
Then, because $R(f)=\mathbb{E}[\hat{R}(f)]$, the following holds:
\begin{align}
    &\mathbb{E}\left[\sup_{f \in \mathcal{F}}\left(\hat{R}(f)-R(f)\right)\right]\nonumber\\
    &=\mathbb{E}_{\mathcal{S}}\left[\sup_{f \in \mathcal{F}}\left(
    \frac{1}{n}\sum_{i=1}^{n}\widetilde{\mathcal{L}}_{I}(f(x_i), \tilde{y}_i)
    -\mathbb{E}_{\mathcal{S}'}\left[\frac{1}{n}\sum_{i=1}^{n}\widetilde{\mathcal{L}}_{I}(f(x'_i), \tilde{y}'_i)\right]\right)\right]\nonumber\\
    &\le\mathbb{E}_{\mathcal{S}}\mathbb{E}_{\mathcal{S}'}\left[\sup_{f \in \mathcal{F}}\left(
    \frac{1}{n}\sum_{i=1}^{n}\widetilde{\mathcal{L}}_{I}(f(x_i), \tilde{y}_i)
    -\frac{1}{n}\sum_{i=1}^{n}\widetilde{\mathcal{L}}_{I}(f(x'_i), \tilde{y}'_i)\right)\right]\nonumber\\
    &=\mathbb{E}_{\mathcal{S}}\mathbb{E}_{\mathcal{S}'}\left[\sup_{f \in \mathcal{F}}\left(
    \frac{1}{n}\sum_{i=1}^{n}\widetilde{\mathcal{L}}_{I}(f(x_i), \tilde{y}_i)
    -\frac{1}{n}\sum_{i=1}^{n}\widetilde{\mathcal{L}}_{I}(f(x'_i), \tilde{y}'_i)\right)\right]\nonumber\\
    &=\mathbb{E}_{\mathcal{S}}\mathbb{E}_{\mathcal{S}'}\left[\sup_{f \in \mathcal{F}}\frac{1}{n}\sum_{i=1}^{n}\left(
    \widetilde{\mathcal{L}}_{I}(f(x_i), \tilde{y}_i)
    -\widetilde{\mathcal{L}}_{I}(f(x'_i), \tilde{y}'_i)\right)\right]\nonumber\\
    &=\mathbb{E}_{\mathcal{S}}\mathbb{E}_{\mathcal{S}'}\mathbb{E}_{\sigma}\left[\sup_{f \in \mathcal{F}}\frac{1}{n}\sum_{i=1}^{n}\sigma_{i}\left(
    \widetilde{\mathcal{L}}_{I}(f(x_i), \tilde{y}_i)
    -\widetilde{\mathcal{L}}_{I}(f(x'_i), \tilde{y}'_i)\right)\right]\label{eq:EsupR-R_ii_1}\\
    &\le\mathbb{E}_{\mathcal{S}}\mathbb{E}_{\sigma}\left[\sup_{f \in \mathcal{F}}\frac{1}{n}\sum_{i=1}^{n}\sigma_{i}\widetilde{\mathcal{L}}_{I}(f(x_i), \tilde{y}_i)\right]
    +\mathbb{E}_{\mathcal{S}'}\mathbb{E}_{\sigma}\left[\sup_{f \in \mathcal{F}}\frac{1}{n}\sum_{i=1}^{n}(-\sigma_{i})\widetilde{\mathcal{L}}_{I}(f(x'_i), \tilde{y}'_i)\right]\nonumber\\
    &=2\mathfrak{R}_n(\widetilde{\mathcal{H}})\nonumber\\
    &\le \frac{\sqrt{2}\rho K^2(K-1)}{I(K-I)}\sum_{y\in \mathcal{Y}} \mathfrak{R}_n(\mathcal{G}_y).\label{eq:EsupR-R_ii_2}
\end{align}
In Eq.\ \eqref{eq:EsupR-R_ii_1}, we used the fact that $\widetilde{\mathcal{L}}_{I}(f(x_i), \tilde{y}_i)-\widetilde{\mathcal{L}}_{I}(f(x'_i), \tilde{y}'_i)$ and $\sigma_i(\widetilde{\mathcal{L}}_{I}(f(x_i), \tilde{y}_i)-\widetilde{\mathcal{L}}_{I}(f(x'_i), \tilde{y}'_i))$ follow the same distribution.
For Eq.\ \eqref{eq:EsupR-R_ii_2}, we used Eq.\ \eqref{eq:rade_H-i-i}.
Thus, the following holds for the classification error $\mathcal{E}_{I}$:
\begin{align*}
    \mathcal{E}_{I} &:= R(\hat{f})-R(f^*) \\
    &= (\hat{R}(\hat{f})-\hat{R}(f^*)) + (R(\hat{f})-\hat{R}(\hat{f})) + (\hat{R}(f^*)-R(f^*)) \\
    &\le 2\sup_{f \in \mathcal{F}}|\hat{R}(f)-R(f)| \\
    &\le \begin{cases}
        \displaystyle\frac{\sqrt{2}\rho K^2(K-1)}{I(K-I)}\sum_{y\in \mathcal{Y}} \mathfrak{R}_n(\mathcal{G}_y) + \frac{K(K-1)}{2(K-I)}C_{L}\sqrt{\frac{2\log{(2/\delta})}{n}} & \text{if $I \ge \displaystyle\frac{K}{2}$} \nonumber \\
        \displaystyle\frac{\sqrt{2}\rho K^2(K-1)}{I(K-I)}\sum_{y\in \mathcal{Y}} \mathfrak{R}_n(\mathcal{G}_y) + \frac{K(K-1)}{2I}C_{L}\sqrt{\frac{2\log{(2/\delta})}{n}} & \text{if $I < \displaystyle\frac{K}{2}$}
       \end{cases}.
\end{align*}

\bibliographystyle{apa}
\bibliography{mybibliography}

\begin{thebibliography}{}

\bibitem[\protect\astroncite{Bao et~al.}{2018}]{bao2018classification}
Bao, H., Niu, G., and Sugiyama, M. (2018).
\newblock Classification from pairwise similarity and unlabeled data.
\newblock In {\em Proc.\ of International Conference on Machine Learning},
  pages 452--461.

\bibitem[\protect\astroncite{Cao and Xu}{}]{Cao2020MultiComplementaryAU}
Cao, Y. and Xu, Y.
\newblock Multi-complementary and unlabeled learning for arbitrary losses and
  models.
\newblock {\em Pattern Recognition}.

\bibitem[\protect\astroncite{Cour et~al.}{2011}]{cour2011learning}
Cour, T., Sapp, B., and Taskar, B. (2011).
\newblock Learning from partial labels.
\newblock {\em The Journal of Machine Learning Research}, pages 1501--1536.

\bibitem[\protect\astroncite{Feng et~al.}{2020a}]{pmlr-v119-feng20a}
Feng, L., Kaneko, T., Han, B., Niu, G., An, B., and Sugiyama, M. (2020a).
\newblock Learning with multiple complementary labels.
\newblock In {\em Proc.\ of International Conference on Machine Learning},
  pages 3072--3081.

\bibitem[\protect\astroncite{Feng et~al.}{2020b}]{feng2020provably}
Feng, L., Lv, J., Han, B., Xu, M., Niu, G., Geng, X., An, B., and Sugiyama, M.
  (2020b).
\newblock Provably consistent partial-label learning.
\newblock In {\em Proc.\ of Advances in Neural Information Processing Systems},
  pages 10948--10960.

\bibitem[\protect\astroncite{Ishida et~al.}{2017}]{ishida2017nips}
Ishida, T., Niu, G., Hu, W., and Sugiyama, M. (2017).
\newblock Learning from complementary labels.
\newblock In {\em Proc.\ of Advances in Neural Information Processing Systems},
  pages 5639--5649.

\bibitem[\protect\astroncite{Ishida et~al.}{2019}]{ishida2019icml}
Ishida, T., Niu, G., Menon, A., and Sugiyama, M. (2019).
\newblock Complementary-label learning for arbitrary losses and models.
\newblock In {\em Proc. \ of International Conference on Machine Learning},
  pages 2971--2980.

\bibitem[\protect\astroncite{Ishiguro et~al.}{2022}]{Ishiguro2022noisycon}
Ishiguro, H., Ishida, T., and Sugiyama, M. (2022).
\newblock Learning from noisy complementary labels with robust loss functions.
\newblock In {\em IEICE Transactions on Information and Systems}, pages
  364--376.

\bibitem[\protect\astroncite{Katsura and Uchida}{2020}]{katsura2020bridging}
Katsura, Y. and Uchida, M. (2020).
\newblock Bridging ordinary-label learning and complementary-label learning.
\newblock In {\em Proc.\ of Asian Conference on Machine Learning}, pages
  161--176.

\bibitem[\protect\astroncite{Katsura and Uchida}{2021}]{katsura2021candidate}
Katsura, Y. and Uchida, M. (2021).
\newblock A generalization of ordinary-label learning and complementary-label
  learning.
\newblock {\em SN Computer Science}, 2(4).

\bibitem[\protect\astroncite{Kingma and Ba}{2015}]{kingma2015}
Kingma, D.~P. and Ba, J. (2015).
\newblock Adam: A method for stochastic optimization.
\newblock In {\em Proc. \ of International Conference on Learning
  Representations}.

\bibitem[\protect\astroncite{Maurer}{2016}]{10.1007/978-3-319-46379-7_1}
Maurer, A. (2016).
\newblock A vector-contraction inequality for rademacher complexities.
\newblock In {\em Proc.\ of Algorithmic Learning Theory}, pages 3--17.

\bibitem[\protect\astroncite{Mohri et~al.}{2012}]{mohri2012}
Mohri, M., Rostamizadeh, A., and Talwalkar, A. (2012).
\newblock {\em Foundations of Machine Learning}.
\newblock The MIT Press.

\bibitem[\protect\astroncite{Shimada et~al.}{2021}]{shimada2021classification}
Shimada, T., Bao, H., Sato, I., and Sugiyama, M. (2021).
\newblock Classification from pairwise similarities/dissimilarities and
  unlabeled data via empirical risk minimization.
\newblock {\em Neural Computation}, pages 1234--1268.

\bibitem[\protect\astroncite{Wen et~al.}{2021}]{pmlr-v139-wen21a}
Wen, H., Cui, J., Hang, H., Liu, J., Wang, Y., and Lin, Z. (2021).
\newblock Leveraged weighted loss for partial label learning.
\newblock In {\em Proc. \ of International Conference on Machine Learning},
  pages 11091--11100.

\bibitem[\protect\astroncite{Yu et~al.}{2018}]{yu2018learning}
Yu, X., Liu, T., Gong, M., and Tao, D. (2018).
\newblock Learning with biased complementary labels.
\newblock In {\em Proc.\ of European Conference on Computer Vision}, pages
  68--83.

\end{thebibliography}




\end{document}